\documentclass[12pt, a4paper]{article}

\usepackage[latin1]{inputenc}
\usepackage{pdfsync}
\usepackage{hyperref}
\usepackage{enumerate} 
\usepackage{multicol}
\usepackage{amsfonts, amsmath, amssymb, amsthm}
\usepackage[ruled]{algorithm2e}
\usepackage{graphicx}
\usepackage{xcolor}
\usepackage{subcaption} 
\usepackage{float}
\usepackage{multirow}
\usepackage{bbold}
\usepackage{geometry}
\numberwithin{equation}{section}
\theoremstyle{plain}
\newtheorem{thm}{Theorem}[section]
\newtheorem{coro}[thm]{Corollary}

\newtheorem{lem}[thm]{Lemma}

\theoremstyle{definition}

\theoremstyle{remark}
\newtheorem{rem}[thm]{Remark}

\newcommand{\R}{{\mathbb R}}

\newcommand{\Z}{{\mathbb Z}}

\newcommand{\be}{\begin{equation}}
\newcommand{\ee}{\end{equation}}
\newcommand{\calP}{{\cal P}}
\newcommand{\calM}{{\cal M}}
\newcommand{\eps}{{\varepsilon}}

\title{Genetic column generation: \\ Fast computation of high-dimensional multi-marginal optimal transport problems}

\author{\normalsize
Gero Friesecke\thanks{Faculty of Mathematics, Technische Universit\"at M\"unchen,
\texttt{gf@ma.tum.de}} \ \ \  Andreas S. Schulz\thanks{Operations Research, Technische Universit\"at M\"unchen,
\texttt{andreas.s.schulz@tum.de}} \ \ \  Daniela V\"{o}gler\thanks{Faculty of Mathematics, Technische Universit\"at M\"unchen,
\texttt{voegler@ma.tum.de}}
}
\date{}

\begin{document}
\maketitle
\vspace{-5mm}
\begin{abstract} We introduce a simple, accurate, and extremely efficient method for numerically solving the multi-marginal optimal transport (MMOT) problems arising in density functional theory. 
The method relies on (i) the sparsity of optimal plans [for $N$ marginals discretized by $\ell$ gridpoints each, general Kantorovich plans require $\ell^N$ gridpoints but the support of optimizers is of size $O(\ell\cdot N)$ \cite{FV18}], (ii) the method of column generation (CG) from discrete optimization which to our knowledge has not hitherto been used in MMOT, and (iii) ideas from machine learning. \textcolor{black}{The well-known bottleneck in CG consists in generating new candidate columns efficiently; we prove that in our context, finding the best new column is an NP-complete problem. 
To overcome this bottleneck} we use a genetic learning method tailormade for MMOT
in which the dual state within CG plays the role of an ``adversary'', in loose similarity to Wasserstein GANs. 
On a sequence of benchmark problems with up to 120 gridpoints and up to 30 marginals, our method always found the exact optimizers. Moreover, empirically the number of computational steps needed to find them appears to scale only polynomially when both $N$ and $\ell$ are simultaneously increased (while keeping their ratio fixed to mimic a thermodynamic limit of the particle system).

 \end{abstract}
 

\section{Introduction} \label{sec:Intro}
Multi-marginal optimal transport (MMOT) suffers from the curse of dimension. If the marginals are discretized by $\ell$ gridpoints, optimal (or candidate) Kantorovich plans for the $N$-marginal problem require $\ell^N$ gridpoint values. While powerful and successful computational schemes centered around the Sinkhorn algorithm have been developed for two-marginal problems ($N=2$) \cite{Cu13,Sch16,Sch19,PC19,BS20}, with recent extensions to a small number of marginals \cite{BCN16,Ne17,BCN19}, the high-dimensionality of multi-marginal plans forbids the use of these schemes in practice already beyond a handful of marginals.

On the other hand, in recent applications of MMOT to many-electron physics \cite{CFK13, BDG12}, data science \cite{AC11}, or fluid dynamics \cite{Br89, Ne17}, $N$ corresponds, respectively, to the number of electrons in a molecule, datasets in a database, or timesteps. 
This makes it highly desirable to develop computational schemes for MMOT with large $N$. \textcolor{black}{In the context of the MMOT problem arising in many-electron physics \cite{CFK13, BDG12} which is our key motivating application and the focus of this paper, some recent advances were made. In \cite{FV18} two of the present authors obtained a rigorous sparsity result (whose ancestor is the celebrated Brenier's theorem \cite{Br91}): after discretization, for any marginals and costs there exist optimizers which are superpositions of at most $\ell$ symmetrized Dirac measures. Moreover the structure of optimizers was shown to be closely related to the Monge ansatz of OT theory, and a two-marginal formulation of the $N$-marginal problem was given. In \cite{KY19}, Khoo and Ying introduced and studied a semi-definite relaxation of the two-marginal formulation and presented an algorithm for the relaxed problem. In \cite{ACEL21, ACE21}, Alfonsi, Coyaud, Ehrlacher, and Lombardi established existence of sparse optimizers even in the situation when the state space is kept continuous and only the marginal constraints are discretized; moreover they proposed a constrained Lagrangian particle method for the ensuing problem. Also, let us mention a recent advance not related to MMOT, namely that smooth {\it two}-marginal problems in high dimension are soluble with dimension-free computational rates, with potentially exponentially dimension-dependent constants \cite{Via21}.}

Here we present a simple and extremely efficient algorithm for MMOT which combines MMOT sparsity, methods from high-dimensional discrete optimization, and recent advances in machine learning. \textcolor{black}{Numerical results show that it allows the accurate} computation of optimal plans with,  say, $N=30$ marginals and $\ell=100$ gridpoints or basis functions per dimension (i.e., $\ell^N=10^{60}$) with Matlab on a laptop. In benchmark examples of this size where the exact solution is known, the algorithm always found the exact optimizers (see section \ref{sec:numerical}). Moreover, empirically (see Figure 5) the number of computational steps needed to find them scales only polynomially instead of exponentially in the thermodynamic limit when both $N$ and $\ell$ get large with their ratio $N/\ell$ remaining constant, although we cannot offer a rigorous proof of this fact. Instead, in section \ref{sec:NP} we show that the pricing problem which our genetic learning method addresses is NP-complete. For a related result recently posted on arXiv see \cite{AB20}. 

Our algorithm, which we call {\it Genetic Column Generation} ({\tt GenCol}), is presented in this paper in detail in the context of the multi-marginal optimal transport problems arising in many-electron physics. 
It is based on three ideas:
\begin{itemize}
\item the existence of extremely sparse optimizers as first pointed out and investigated in the present context by two of the authors in \cite{FV18}. This breaks the curse of dimension with respect to storage complexity (but at the time we could not offer any algorithm).
\item the method of column generation (CG), which is well established in discrete optimization but has to our knowledge not hitherto been used in MMOT. CG is a pragmatic approach to tackle certain extremely high-dimensional problems which originated in integer programming. \textcolor{black}{The latter arises when looking for Monge plans for $\ell=2$ and $N$ large, in which case the unknown is a pair of vectors in $\{0,1\}^N$.} We note that this is exactly the opposite regime to $N=2$, $\ell$ large where the Sinkhorn algorithm works most successfully.
\item a simple genetic method tailormade for MMOT to overcome the well known bottleneck in CG that one must be able to generate new candidate columns efficiently. In our context new columns represent intricate spatial many-body correlation patterns of the system which are not known a priori; these are learned with the help of an ``adversary'' represented by the dual state within CG, in loose similarity to Wasserstein GANs \cite{ACB17}.    
\end{itemize}
The underlying theory is described in sections \ref{sec:MMOT}--\ref{sec:genetic}. The algorithm, which in the end is rather simple, is presented in section \ref{sec:algo}. Numerical results for test problems up to sizes of $\ell^N\approx 10^{60}$ are given in section \ref{sec:numerical}. 
\textcolor{black}{Applications to more complex electronic structure problems} will be given elsewhere.

\section{MMOT, motivation, discretization} \label{sec:MMOT}

{\bf Multi-marginal optimal transport.} Many different problems in mathematics, science, and engineering can be cast in the form of the general multi-marginal optimal transport problem: 

\begin{center}
\begin{minipage}{0.85\textwidth}
Minimize a cost functional 
\begin{equation} \label{MMOT}
          C[\gamma ] = \int_{X_1\times \ldots \times X_N}  c(x_1,\ldots,x_N) \, d\gamma(x_1,\ldots,x_N)
\end{equation}
over $N$-point probability measures
\begin{equation} \label{space}
     \gamma\in\calP(X_1\times \ldots \times X_N)
\end{equation}
subject to the marginal constraints
\begin{equation} \label{margs}
          M_{X_i}\gamma = \mu_i \;\;\; (i=1,\ldots,N).
\end{equation}
\end{minipage}
\end{center}

\noindent
Here the $X_i$ are metric spaces (in practice, subsets of $\R^d$ for continuous problems and finite sets for discrete problems), the $\mu_i$ are given Borel probability measures on $X_i$,  ${\calP}(X_1\times \ldots \times X_N)$ denotes the set of Borel probability measures on $X_1\times \ldots \times X_N$, $c \, : \, X_1\times \ldots \times X_N \to \R\cup\{+\infty\}$ is a cost function, and the marginal of $\gamma$ with respect to the i$^{th}$ space $X_i$ is the probability measure on $X_i$ defined by 
$$
          M_{X_i}\gamma (A) = \gamma(X_1\times \ldots \times X_{i-1}\times A \times X_{i+1} \times \ldots \times X_N) \; \mbox{ for all measurable }A\subseteq X_i.
$$
Optimizers are known as {\it optimal plans} or {\it Kantorovich plans}. 
Both the analysis and the numerical treatment of optimal transport problems have been the subject of intensive and fruitful research, with the focus overwhelmingly on two-marginal problems ($N=2$); see \cite{Vi09, Sa15, PC19} for wide-ranging surveys. 

Multi-marginal problems ($N>2$), about which much less is known, have been considered for quite some time in operations research, probability theory, analysis, and mathematical economics \cite{Pi68, Po94, RR98, GS98, Sp00, CMN10, BDM12}. Recently, important examples of multi-marginal problems with large $N$ have emerged independently in many-electron physics \cite{CFK13, BDG12}, fluid dynamics \cite{Br89, Ne17}, and data science \cite{AC11}. The number $N$ of marginals corresponds, respectively, to the number of particles, timesteps, or datasets in a database, 
motivating the interest in large $N$. 

{\bf Physical motivation.} A central example which we want to attack in this paper is multi-marginal optimal transport with Coulomb cost, which arises as the strongly correlated limit of density functional theory (DFT). DFT is the most widely used method for numerical electronic structure computations in physics, chemistry, and materials science, see \cite{Be14} for a review. The strongly correlated limit was introduced by Seidl \cite{Se99}. As first noticed and exploited in \cite{CFK13, BDG12} the limit problem is an optimal transport problem, with
\begin{equation} \label{CoulombCase}
   X_i = \bar{X}\subseteq\R^d \, \forall i, \;\; \mu_i = \mu \, \forall i, \;\; c(x_1,\ldots,x_N) = \sum_{1\le i <j \le N} \frac{1}{|x_i-x_j|}
\end{equation}
where $\mu \, : \, \R^d\to\R$ is the single-particle density of the system, normalized so that it integrates to $1$. See \cite{CFK18} for a rigorous derivation from the underlying quantum many-body system. In physics one is only interested in Kantorovich plans which are symmetric with respect to the $x_i$ (as these model $N$-point position densities of electrons, which are symmetric by the laws of quantum theory). This means that for all permutations $\sigma$, 
$$
   \gamma(A_1\times \ldots \times A_N) = \gamma(A_{\sigma(1)}\times \ldots \times A_{\sigma(N)})  \mbox{ for any Borel subsets }A_1,\ldots,A_N \mbox{ of }\bar{X}.
$$
Mathematically, this restriction does not alter the optimal cost because for equal marginals and a symmetric cost $c$ (as in \eqref{CoulombCase}), each non-symmetric plan gives rise to a symmetric one with the same cost, by symmetrization. Also, for a symmetric plan, any one marginal condition implies the others. Thus in the situation \eqref{CoulombCase}, denoting the set of symmetric probability measures on $\bar{X}^N$ by $\calP_{sym}(\bar{X}^N)$ and abbreviating $M_{X_1}\gamma=M_1\gamma$, the MMOT problem \eqref{MMOT}--\eqref{margs} reduces to 
\begin{align}
 & \mbox{Minimize } C[\gamma ] = \int_{\bar{X}^N}  c(x_1,\ldots,x_N) \, d\gamma(x_1,\ldots,x_N) \label{MMOT'}
 \\
 & \mbox{over }\gamma\in\calP_{sym}(\bar{X}^N) \label{space'}
 \\[2mm]
 & \mbox{subject to } M_1\gamma = \mu. \label{margs'}
\end{align}
(symmetric MMOT). Here $c$ can be any symmetric function on $\bar{X}^N$.

Corrections from the strongly correlated (multi-marginal optimal transport) limit have been demonstrated to improve the accuracy of electronic structure simulations based on DFT \cite{FGSD16}; but as yet no numerical method is available which can handle this limit reliably for other than small test systems with a few electrons. 
\vspace*{2mm}

{\bf Discretization.} A simple, in the $N=2$ case standard, structure-preserving discretization of \eqref{MMOT}--\eqref{margs} which preserves the favourable sparsity and duality properties of OT is as follows. Suppose the $X_i$ are compact subsets of $\R^d$ and $c \, : \, X_1\times \ldots \times X_N\to\R$ is continuous. Let 
\begin{equation} \label{discmargs}
  \mu_i^{(\nu)} = \sum_{\alpha=1}^{\ell_i(\nu)} m_{i,\alpha}^{(\nu)} \delta_{a_{i,\alpha}^{(\nu)}}, \; m_{i,\alpha}^{(\nu)}\ge 0, \; a_{i,\alpha}^{(\nu)}\in X_i,
\end{equation}
be any sequence of finite sums of Dirac measures converging weak* in ${\cal M}(X_i)=(C(X_i))^*$ to $\mu_i$. (Such approximations always exist. For instance, if $X_i$ is the closure of an open bounded set with smooth boundary, one may partition $X_i$ into distinct small cells $V_{i,\alpha}^{(\nu)}=X_i\cap Q_{i,\alpha}^{(\nu)}$ where the $Q_{i,\alpha}^{(\nu)}$ are disjoint cubes in $\R^d$ of sidelength $1/\nu$. One now picks any  representative point  $a_{i,\alpha}^{(\nu)}$ in $V_{i,\alpha}^{(\nu)}$ and places all the mass from $V_{i,\alpha}^{(\nu)}$ there, i.e. one sets  $m_{i,\alpha}^{(\nu)}=\mu_i(V_{i,\alpha}^{(\nu)})$.) Then any plan $\gamma\in\calP(X_1\times \ldots \times X_N)$ satisfying the marginal conditions \eqref{margs} must be of the following form, where we omit the superscript $\nu$: 
\begin{equation} \label{zero'}
   \gamma = \sum_{i_1=1}^{\ell_1} \ldots \sum_{i_N=1}^{\ell_N} \gamma_{i_1 \ldots i_N}\delta_{a_{1,i_1}} \otimes \ldots \otimes \delta_{a_{N,i_N}}
\end{equation}
so such a plan can be viewed as a tensor $(\gamma_{i_1 \ldots i_N})$ of order $N$ and \eqref{MMOT}--\eqref{margs} reduces to the discrete problem 
\begin{align}
  & \mbox{Minimize }C[\gamma ] = \sum_{i_1=1}^{\ell_1}\ldots \sum_{i_N=1}^{\ell_N} \gamma_{i_1\ldots i_N} c(a_{1,i_1},\ldots,a_{N,i_N}) \label{one'} \\
  & \mbox{subject to } \sum_{i_j\, : \, j\neq k} \gamma_{i_1\ldots i_k\ldots  i_N} = m_{k,i_k} \; \forall i_k\in\{1,\ldots ,\ell_k\} \label{two'} \\
  & \mbox{\textcolor{white}{subject to }}\gamma\ge 0 \label{three'}
\end{align}
(with the last inequality understood componentwise). For symmetric MMOT, we may assume that the $\ell_i$, $a_{i,\alpha}$, and $m_{i,\alpha}$ are independent of $i$, and the discrete problem  reads as follows: given a set of $\ell$ distinct discretization points,
\begin{equation} \label{eq:FinSta}
        X = \{ a_1,\ldots ,a_\ell\} \subset\R^d,
\end{equation}
and a marginal $\lambda^*\in\calP(X)$ which we may view as a vector in $\R^\ell$ whose $i^{th}$ component is given by $\lambda^*(\{a_i\})$, 
\begin{align} 
  & \mbox{Minimize }C[\gamma ] = \sum_{i_1,\ldots ,i_N=1}^{\ell} \gamma_{i_1\ldots i_N} c(a_{i_1},\ldots ,a_{i_N}) \mbox{ over }\gamma\in\calP_{sym}(X^N) \label{one''} \\[-1mm] 
  & \mbox{subject to }\sum_{i_2,\ldots ,i_N=1}^\ell \gamma_{i_1i_2\ldots i_N}= \lambda^*_{i_1} \mbox{ for all }i_1\in\{1,\ldots ,\ell\} \label{two''} \\[1mm]
  & \mbox{\textcolor{white}{subject to }}\gamma \ge 0. \label{three''} 
\end{align}
\textcolor{black}{The associated dual problem is 
\begin{align}
  & \mbox{Maximize } \sum_{i=1}^\ell y_i \lambda_i^* \mbox{ over }y\in\R^\ell \label{four''} \\
  & \mbox{subject to }\tfrac{1}{N}\bigl(y_{i_1}+\ldots + y_{i_N}\bigr) \le c(a_{i_1},...,a_{i_N}) \; \forall i_1,\ldots,i_N\in\{1,\ldots ,\ell\}; \label{five''}
\end{align}
it discretizes the continuous dual problem \cite{BDG12} to maximize $\int_{\bar{X}} y \, d\mu$ over measurable functions $y\, : \, \bar{X}\to\R$ satisfying $\tfrac{1}{N}(y(x_1)+ \ldots +y(x_N))\le c(x_1,\ldots ,x_N)$ $\forall x_1,\ldots, x_N\in \bar{X}$, whose solution are called {\it Kantorovich potentials}. By LP duality, the value of \eqref{one''}--\eqref{three''} equals that of \eqref{four''}--\eqref{five''}.}

Application of a well known stability result in optimal transport theory (see \cite{Sa15} Theorems 1.50 and 1.51 in the context of two-marginal problems; the extension to $N$ marginals is straightforward) immediately yields the following convergence result as $\nu\to\infty$.
\begin{thm} \label{T:disc_to_cont} (Justification of discretization) For any compact sets $X_1,\ldots ,X_N$ in $\R^d$, any continuous cost $c\, : \, X_1\times \ldots  \times X_N\to\R$, and any discretization \eqref{discmargs} of the marginals which converges weak* to these, the optimal cost of  the discretized problem \eqref{zero'}--\eqref{two'} converges to that of the continuous problem \eqref{MMOT}--\eqref{margs}. Moreover any sequence of optimizers $\gamma^{(\nu)}$ of the discretized problem converges -- after passing to a subsequence -- weak* to a minimizer of the continuous problem.  
\end{thm}

More sophisticated discretizations can be considered. For instance one can represent integrable marginals $\mu_k$ by piecewise linear finite elements and use effective cost coefficients obtained by integrating the continuous cost function against the tensor products of these elements, as in \cite{CFM15} where the Coulomb problem was simulated for the dihydrogen molecule. 
For smooth marginals and costs this is expected to improve the discretization error from $O(1/\nu)$ to $O(1/\nu^2)$. Moreover, to alleviate the computational cost the elements could be chosen adaptively so that each element carries approximately the same marginal mass \cite{CFM15}. In this paper we do not investigate such refinements, and confine ourselves to the basic qualitative justification of the discretization \eqref{zero'}--\eqref{two'} given in Theorem \ref{T:disc_to_cont}.

\textcolor{black}{The discrete problems \eqref{one'}--\eqref{three'} or \eqref{one''}--\eqref{three''} are high-dimensional LPs. For some costs with very special interaction structure (such as the Wasserstein barycenter problem) a transformation to low-dimensional LPs is possible \cite{COO15} (see also \cite{AlB20}), making standard methods from linear programming applicable. For general costs, including \eqref{CoulombCase}, such schemes become unfeasible beyond a handful of marginals, due to the curse of dimension.}

\section{Extremal formulation of symmetric multi-marginal OT} \label{sec:extremal}
Starting point of the algorithm presented here is the following equivalent formulation of symmetric MMOT introduced in \cite{FV18}, in which (candidate and optimal) Kantorovich plans are expressed as convex combinations of extreme points of ${\cal P}_{sym}(X^N)$. This eliminates any redundancy in the parametrization of plans and thus reduces the problem dimension, while at the same time keeping the problem in the form used in two of the pioneering articles on column generation \cite{DW60, DW61}. 

It is not difficult to show (see \cite{FV18}) that when $X$ is a finite state space, \eqref{eq:FinSta}, the extreme points of $\calP_{sym}(X^N)$ can be uniquely recovered from their marginals, which are given by the $\tfrac{1}{N}$-quantized probability measures on $X$,
\begin{equation}
\label{eq:QuantizedProbMeasures}
  \calP_ {\frac{1}{N}}(X):=\left\{ \lambda\in \calP(X)\, \big| \, \lambda(\{ a_i\})
\in\left\{0,\frac{1}{N},\frac{2}{N},\ldots,\frac{N}{N}\right\} \textrm{ for all } i \in \{ 1,\ldots,\ell \}\right\} .
\end{equation}
To recover the corresponding extreme point, write an element $\lambda$ from the above set in the form $\sum_{k=1}^N\tfrac{1}{N}\delta_{a_{i_k}}$ for some (not necessarily distinct) points $a_{i_1},\ldots ,a_{i_N}\in X$ and set
\begin{equation} \label{Nptconfig}
     \gamma_\lambda = S_N \delta_{a_{i_1}}\otimes \ldots  \otimes \delta_{a_{i_N}}.
\end{equation}
Here $S_N$ is the symmetrizer defined by $(S_N\gamma)(A_1\times \ldots  \times A_N)=\tfrac{1}{N!}\sum_\sigma \gamma(A_{\sigma(1)}\times \ldots  \times A_{\sigma(N)})$, with the sum running over all permutations of $\{1,\ldots ,N\}$. Of course any element of the set $\calP_{sym}(X^N)$ is a convex combination of extreme points, but here something better is true:
\begin{lem} \label{L:extremal_rep} \cite{FV18} Any element $\gamma\in\calP_{sym}(X^N)$ can be represented uniquely as a convex combination of the above extreme points, that is, 
\begin{equation} \label{deFi}
          \gamma = \sum_{\lambda\in\calP_{\frac{1}{N}}(X)} \alpha_\lambda \gamma_\lambda, \;\;
          \alpha_\lambda\ge 0 \; \forall \lambda\in\calP_{\frac{1}{N}}(X), \;\; \sum_{\lambda\in\calP_{1/N}(X)}\alpha_\lambda = 1.
\end{equation}
\end{lem}
Here the uniqueness is obvious from the fact that the $\gamma_\lambda$ have mutually disjoint support. 

Since $\gamma_\lambda$ has marginal $\lambda$, the marginal condition becomes
\begin{equation}
\label{eq:ATSPedestrianTwo}
\lambda^*=\sum_{\lambda\in\mathcal{P}_ {\frac{1}{N}}(X)} \alpha_{\lambda} \lambda.
\end{equation}
Thus the MMOT problem \eqref{zero'}--\eqref{three'} can be written as the following optimization problem over the coefficient vectors $\alpha$. Here and below we identify probability measures $\lambda\in\calP_{1/N}(X)$ with vectors in $\R^\ell$ whose $i$-th component is given by $\lambda(\{a_i\})$.
\begin{align}
\label{eq:ColGenProblemOne}	
\textrm{Minimize }		& c^T\alpha =\sum_{\lambda\in\mathcal{P}_{\frac{1}{N}}(X)} c_{\lambda}\alpha_{\lambda} \\
\label{eq:ColGenProblemTwo}
\textrm{subject to }	& A\alpha =\lambda^*\\
\label{eq:ColGenProblemThree}
						& \alpha \geq  0,
\end{align}
with cost coefficients
\begin{equation} \label{cost1}
c_{\lambda}= \sum_{i_1,\ldots ,i_N=1}^\ell (\gamma_\lambda )_{i_1,\ldots ,i_N} c(a_{i_1},\ldots ,a_{i_N})
\end{equation}
and $A$ being the $\ell\times\binom{N+\ell-1}{N}$ matrix defined by 
\begin{equation} \label{A}
A\alpha=\sum_{\lambda\in\mathcal{P}_{\frac{1}{N}}(X)} \alpha_{\lambda} \lambda,
\end{equation}
that is, the columns of $A$ are given -- say, in alphabetical order -- by the vectors in $\calP_{1/N}(X)$. For instance, for $\ell=5$ and $N=3$, 
\setcounter{MaxMatrixCols}{20}
\begin{equation*}
  A = \frac{1}{3} \mbox{ $\begin{pmatrix}
  3 \, 0 \, 0 \, 0 \, 0 & 2 \, 2 \, 2 \, 2 \, 1 \, 1 \, 1 \, 1 \, 0 \, 0 \, 0 \, 0 \, 0 \, 0 \, 0 \, 0 \, 0 \, 0 \, 0 \, 0 & 1 \, 1 \, 1 \, 1 \, 1 \, 1 \, 0 \, 0 \, 0 \, 0 \\
  0 \, 3 \, 0 \, 0 \, 0 & 1 \, 0 \, 0 \, 0 \, 2 \, 0 \, 0 \, 0 \, 2 \, 2 \, 2 \, 1 \, 1 \, 1 \, 0 \, 0 \, 0 \, 0 \, 0 \, 0 & 1 \, 1 \, 1 \, 0 \, 0 \, 0 \, 1 \, 1 \, 1 \, 0 \\
  0 \, 0 \, 3 \, 0 \, 0 & 0 \, 1 \, 0 \, 0 \, 0 \, 2 \, 0 \, 0 \, 1 \, 0 \, 0 \, 2 \, 0 \, 0 \, 2 \, 2 \, 1 \, 1 \, 0 \, 0 & 1 \, 0 \, 0 \, 1 \, 1 \, 0 \, 1 \, 1 \, 0 \, 1 \\
  0 \, 0 \, 0 \, 3 \, 0 & 0 \, 0 \, 1 \, 0 \, 0 \, 0 \, 2 \, 0 \, 0 \, 1 \, 0 \, 0 \, 2 \, 0 \, 1 \, 0 \, 2 \, 0 \, 2 \, 1 & 0 \, 1 \, 0 \, 1 \, 0 \, 1 \, 1 \, 0 \, 1 \, 1 \\
  0 \, 0 \, 0 \, 0 \, 3 & 0 \, 0 \, 0 \, 1 \, 0 \, 0 \, 0 \, 2 \, 0 \, 0 \, 1 \, 0 \, 0 \, 2 \, 0 \, 1 \, 0 \, 2 \, 1 \, 2 & 0 \, 0 \, 1 \, 0 \, 1 \, 1 \, 0 \, 1 \, 1 \, 1 \end{pmatrix}$} .  
\end{equation*} 
Note that the normalization condition that the $\alpha_\lambda$ must sum to $1$ is automatically enforced by the marginal constraints $A\alpha=\lambda^*$. 

We refer in the sequel to the linear program eq.~\eqref{eq:ColGenProblemOne}--\eqref{eq:ColGenProblemThree} as the {\it master problem} (MP). This is the problem we seek to tackle in this paper. Note that the curse of dimension is still present as the number of unknowns still grows combinatorially in $N$; just that by exploiting symmetry we have reduced it from $\ell^N$ in \eqref{zero'}--\eqref{two'} to $\binom{N+\ell-1}{N}$. For instance, for $25$ particles and $100$ gridpoints for discretizing the marginal, this reduces the number of unknowns from $10^{50}$ to about $10^{26}$ -- still out of reach of conventional methods.

\section{Sparsity of optimizers; sparse manifolds}
A fundamental feature of the above MP which our algorithm exploits is the extreme sparsity of optimizers. As is well known in polyhedral optimization, the number of nonzero entries of extremal optimizers is governed by the number of equality constraints. In the context of MMOT, this number is much smaller than the number of unknowns, and the ensuing exact sparse ansatz was first introduced and investigated by two of the authors in \cite{FV18}, where the following result was proved.
\begin{thm} \label{T:sparse} \cite{FV18} For any $\ell$ and $N$, any symmetric cost function $c \, : \, X^N\to\R$, and any marginal $\lambda_*\in\calP(X)$, there exists an optimizer $\alpha_*$ of \eqref{eq:ColGenProblemOne}--\eqref{A} belonging to the manifold
$$
   \calM_\ell := \{ \alpha\in\R^{\binom{N+\ell-1}{N}} \, \big| \, \alpha_\lambda \ge 0 \, \forall \lambda, \; A\alpha = \lambda^*, \; \alpha  
   \mbox{ has at most $\ell $ nonzero entries} \}.
$$
\end{thm}
Thus in our case the number of required nonzero entries is just $\ell$, independently of $N$. (Strictly speaking $\calM_\ell$ is not a manifold but only an algebraic variety.)

In \cite{FV18} we proposed the name {\it quasi-Monge states} for the elements in this sparse manifold, because of a close connection with the Monge ansatz in optimal transport. More precisely, one can show \cite{FV18} that each plan $\gamma$ corresponding to a coefficient vector in $\calM_\ell$ can be written in the form
$$
   \gamma = S_N \sum_{\nu=1}^\ell \mu_\nu \delta_{T_1(a_\nu)}\otimes \ldots  \otimes \delta_{T_N(a_\nu)}
$$
or, in optimal transport notation (with $( \; )_\sharp$ denoting the push-forward of a measure)
$$
   \gamma = S_N (T_1,\ldots ,T_N)_\sharp \mu,
$$
for $N$ maps $T_1,\ldots ,T_N\, : \, X\to X$ and $\ell$ coefficients $\mu_{\nu}\ge 0$ which sum to $1$. Restricting $\mu$ to be equal to the prescribed marginal $\lambda^*$ is the classical Monge ansatz from optimal transport theory. But the latter is too restrictive for the validity of Theorem \ref{T:sparse} when $N\ge 3$, even in the case of the uniform marginal $\lambda^*=\tfrac{1}{\ell}\sum_{\nu=1}^\ell \delta_{a_\nu}$ (see \cite{Fr19} for simple counterexamples \textcolor{black}{and \cite{Vo19} for a systematic numerical study}). 

From a computational perspective it will be useful to work on a slightly larger ansatz manifold,  
\begin{equation} \label{ML'}
   \calM_{\ell'} := \{\alpha\in\R^{\binom{N+\ell-1}{N}} \, \big| \, \alpha_\lambda \ge 0 \; \forall \lambda, \, A\alpha = \lambda^*, \,  \alpha_* \mbox{ has at most $\ell'$ nonzero entries} \},
\end{equation} 
where
\begin{equation} \label{L'}
    \ell < \ell' = \ell + O(\ell ).
\end{equation}
In practice we will use
\begin{equation}
   \ell' = \beta \ell 
\end{equation}
where $\beta$ is a hyperparameter in the {\tt GenCol} algorithm (chosen to be $5$ in all our simulations). The intuition behind the enlargement of $\calM_\ell$ to $\calM_{\ell'}$ is that it keeps the sparsity at an extremely low level but makes the problem less nonlinear. (In the -- in practice unfeasible -- limit $\ell'=\binom{N+\ell-1}{N}$ one would obtain back the original linear program.)

%
%
%
%
\section{Genetic column generation} \label{sec:genetic}
\subsection{Column generation}
In light of Theorem \ref{T:sparse} it is -- in principle -- possible to solve the master problem {\it exactly} via an algorithm that runs only on the data-sparse manifolds $\calM_\ell$ or $\calM_{\ell + O(\ell)}$ without ever touching the master problem in its entirety. But what to do in practice? Column generation (CG) is a pragmatic approach from discrete optimization, of primal-dual type, in which the primal state evolves precisely on such a sparse manifold. To the best of our knowledge, CG has not hitherto been considered in connection with optimal transport. Its development originated in integer programming, but it has been especially useful in $0/1$-integer programming, \textcolor{black}{where the unknown is a vector in $\{0,1\}^N$ (corresponding to the domain of Kantorovich plans for $\ell=2$ and $N$ marginals) and} where CG can be used in association with branch-and-bound techniques. Successful applications include traveling salesman problems, airline scheduling, and vehicle routing (see, e.g., \cite{LD05}). 

Consider any linear program of the form of our master problem 
\eqref{eq:ColGenProblemOne}--\eqref{eq:ColGenProblemThree}, and suppose we are in the general situation (satisfied in our case) that the matrix $A$ has far fewer rows than columns and the number of columns is far too large to use standard LP solvers (such as Gurobi, \cite{Gurobi19}). In CG one starts off by reducing the master problem to a problem with far fewer variables by admitting only a small sized subset of the columns of $A$ as new constraint matrix. As only those admissible coefficient vectors $\alpha$ of the MP that are supported on the chosen columns are admissible for the new problem, the MP can be viewed as a ``relaxation'' of this new problem. Now the idea is to suitably  $\textit{generate}$ additional candidate \textit{columns} for the reduced problem and use \textcolor{black}{a duality based criterion} to accept or reject them in order to decrease its optimal value and thereby the gap to  the optimal value of the MP.

Let us now explain the method in detail. The first step in CG consists of choosing a small sized subset $I\subset \bar{I}=\left\{1,2,\ldots,\binom{N+\ell-1}{N}\right\}$ of the columns of the constraint matrix $A$ of the MP \eqref{eq:ColGenProblemOne}-\eqref{eq:ColGenProblemThree}. For any such $I$, $A_I$ and $c_I$ denote the submatrix of $A$, respectively the subvector of $c$ that contains exactly the corresponding columns respectively entries. Replacing the original constraint matrix $A$ of the MP \eqref{eq:ColGenProblemOne}-\eqref{eq:ColGenProblemThree} by $A_I$ and the cost vector $c$ by $c_I$ yields the problem
\begin{align}
\label{eq:RMPOne}	
\textrm{Minimize }		& c_I^T\alpha \\
\label{eq:RMPTwo}
\textrm{subject to }	& A_I\alpha =\lambda^*\\
\label{eq:RMPThree}
						& \alpha \geq  0.
\end{align}
Problem \eqref{eq:RMPOne}-\eqref{eq:RMPThree} will be referred to as the \textit{restricted master problem} (RMP). As long as $\sharp I \le \ell'$, candidate or optimal primal states of the RMP, extended by zero to $\bar{I}$, stay in the sparse manifold $\calM_{\ell'}$, eq.~\eqref{ML'}. 

Given a RMP \eqref{eq:RMPOne}-\eqref{eq:RMPThree} induced by a reduced constraint matrix $A_I$,  one would like to add ``better'' columns to $A_I$, i.e., columns that improve the optimal value of the RMP. These 'better' columns are best understood from a dual point of view. The \textit{dual of the restricted master problem}
(DRMP) is given by 
\begin{align}
\label{eq:DRMPOne}	
\textrm{Maximize }		& y^T\lambda^* \\
\label{eq:DRMPTwo}
\textrm{subject to }	& A_I^T y \leq c_I.
\end{align}
Replacing $A_I$ by $A$ and $c_I$ by $c$ in this problem yields the \textit{dual of the master problem} (DMP). The DRMP differs from the DMP by imposing far fewer constraints ($\sharp I$ instead of $\sharp \bar{I}$) on the dual variables $y\in \mathbb{R}^{\ell}$. 

Theoretical discussions of column generation now continue with the following -- for high-dimensional problems infeasible -- step, in which the dual problem is used to find the ``best'' additional column: given a dual optimal solution $y^*$ of the DRMP, solve the so called \textit{pricing problem} (PP) 
\begin{align}
\label{eq:PPOne}
\textrm{Maximize }		& \lambda^Ty^*-c_{\lambda} \\
\label{eq:PPTwo}
\textrm{subject to }	& \lambda\in \mathcal{P}_{\frac{1}{N}}(X).\end{align}
This problem looks for the constraint of the DMP that is violated the most by the given optimal solution $y^*$ of the DRMP.  

But the pricing problem suffers from the fundamental problem we seek to circumvent, namely the curse of dimension. In fact, we will show in section \ref{sec:NP} that even for pairwise costs, in which case the evaluation of $c_\lambda$ is simple (see section \ref{sec:fastcost}), this problem is NP-complete.  

In practice, for high-dimensional problems one needs to replace \eqref{eq:PPOne}--\eqref{eq:PPTwo} by the following: 
\begin{equation} \label{PP'}
   \mbox{Efficiently find a new column $\lambda$ such that }\lambda^Ty^*-c_\lambda > 0.
\end{equation}
Any such column can be added to the restricted constraint matrix $A_I$. The new column represents a constraint of the full dual (DMP) which the solution $y^*$ to the current DRMP violates. Adding this column to the matrix $A_I$ ``cuts off'' $y^*$ from the optimization domain of the DRMP,  yielding a new dual optimal solution $\tilde{y}^*$. Except in degenerate cases, this also leads to a new primal solution and a decrease in cost. For convenience of the reader we include the well known theoretical justification of the acceptance criterion in \eqref{PP'}. 
\begin{lem} (Justification of acceptance criterion) If $\lambda^Ty^*-c_\lambda \le 0$ for all columns $\lambda$ of the full constraint matrix $A$, then the current dual solution $y^*$ of  the DRMP solves the full dual problem DMP, and the current primal solution $\alpha_I$ of the RMP, extended by zeros, solves the full primal problem MP.
\end{lem}
\begin{proof}
Denote the current primal solution extended by zeros by $\overline{\alpha_I}$. By assumption, $A^Ty^*\le c$, that is, $y^*$ is admissible for the full dual problem. Using, in order of appearance, the definition of $\overline{\alpha_I}$, duality for the RMP, admissiblity of $y^*$, and duality for the full MP gives
$$
  c^T \overline{\alpha_I} = c_I^T \alpha_I = \lambda^* y^* \le 
  \max_{y \, : \, A^T y \le c} \lambda^* y = \min_{\substack{\alpha \, : \, A\alpha = \lambda^* \\ \alpha \ge 0}} c^T \alpha. 
$$
Since $\overline{\alpha_I}$ is admissible for the full primal problem, the assertion follows. 
\end{proof}
\subsection{Fast cost evaluation for candidate columns} \label{sec:fastcost}
A possible additional bottleneck in CG besides the large number of columns can be the cost evaluation of a new column, required by the acceptance criterion in \eqref{PP'}. In the case of MMOT with large $N$, a priori this requires evaluation of a high-dimensional sum, see \eqref{one''}. We now show that, due to the special structure of the extreme points $\gamma_\lambda$ of $\calP_{sym}(X^N)$ and the fact that the costs of interest are of pairwise form, this cost evaluation can in fact be done extremely fast, requiring only an $N$-{\it independent} number of arithmetic operations.

First, it is elementary that whenever $c\, : \, X^N\to\R$ is of pairwise and symmetric form,
\begin{equation} \label{pairwise}
   c(x_1,\ldots ,x_N) = \sum_{1\le i<j\le N} w(x_i,x_j) \mbox{ for some symmetric }w \, : \, X\times X\to\R,
\end{equation}
then for any $\gamma\in\calP_{sym}(X^N)$,
$$
   C[\gamma] = {N \choose 2} \sum_{i,j=1}^\ell (M_2\gamma)_{ij} w(a_i,a_j)
$$
where $M_2\gamma$ is the two-point marginal of $\gamma$, defined by $M_2\gamma(A)=\gamma(A\times X^{N-2})$ for all subsets $A$ of $X^2$. For further discussion of this representation and its usefulness in electronic structure see \cite{FMPCK13}. In case of the extreme points $\gamma_\lambda$, the following explicit formula for the two-point marginal in terms of the one-point marginal $\lambda$ was derived in \cite{FV18}; it shows that on these points the highly non-invertible projection map from $M_2\gamma$ to $M_1\gamma$ can be inverted.
\begin{lem} \cite{FV18} For any $\lambda\in\calP_{1/N}(X)$, 
$$
   M_2\gamma_\lambda =
  \frac{N}{N-1}\lambda\otimes\lambda-\frac{1}{N-1}\sum_{i=1}^\ell\lambda(\{a_i\})     
  \delta_{a_i}\otimes\delta_{a_i}.	
$$
Moreover $\gamma_\lambda$ is the unique element of $\calP_{sym}(X^N)$ with this two-point marginal.
\end{lem}
(In optimal transport notation, the second term equals $-\tfrac{1}{N-1}\left(id,id\right)_{\#}\lambda$.) This immediately yields the following simple expression for the cost $c_\lambda$ of a column $\lambda\in\calP_{1/N}(X)$. Any such $\lambda$ is now again
 regarded as a vector in $\R^\ell$.
\begin{coro} \label{C:fastcost} If $c$ has the pairwise form \eqref{pairwise}, and $\lambda$ is any element of $\calP_{1/N}(X)$, then
\begin{equation} \label{cost2}
  c_\lambda =  \frac{N^2}{2}\lambda^T C \lambda-\frac{N}{2}\textrm{diag}(C)^T\lambda
\end{equation}
with $C=\left(C_{ij}\right)_{i,j=1}^\ell\in \mathbb{R}^{\ell\times\ell}$ defined by 
\begin{equation}
\label{eq:CostReform}
     C_{ij}=w(a_i,a_j).	
\end{equation}
\end{coro}
This reduces cost evaluation to just matrix-vector multiplication with a precomputed matrix of $N$-independent size $\ell$, 
and shows that the acceptance criterion in the pricing problem \eqref{PP'} is 
extremely cheap computationally.
\subsection{Genetic method for generating new columns} \label{sec:GenCol}
To tackle \eqref{PP'}, let us recall the physical meaning of columns in MMOT in the key example of electronic structure. Transport plans $\gamma$ with $N$ marginals correspond to the joint probability density of $N$ electron positions in a continuous $d$-dimensional domain $\Omega\subseteq\R^d$, or on the $\ell$ discretization points of an $\ell$-point discretization $X=\{a_1,\ldots ,a_\ell\}\subset\Omega$ (see section \ref{sec:MMOT}). The columns $\lambda\in\calP_{1/N}(X)$ describe all the possible ``pure'' $N$-particle configurations, obtained by dropping the $N$ electrons on the $\ell$ discretization points (while allowing to multiply occupy sites). The MP \eqref{eq:ColGenProblemOne}--\eqref{A} seeks to determine a stochastic superposition of these electron configurations that minimizes the interaction energy -- prototypically, the mutual Coulomb repulsion -- while fulfilling the marginal constraint. The latter describes the single-electron density, that is, the total occupancy of each site. Finding promising new columns corresponds to guessing good new $N$-particle configurations for the given density and interaction. 

We take the view that guessing such  -- intricately correlated -- configurations from the vast number of possiblities must be {\it learned}. The best available information given a current RMP matrix $A_I$ and a solution $\alpha_I$ to the RMP is the information {\it which columns are successful}, i.e. which ones correspond to a nonzero component of the vector $\alpha_I$. But this is already very valuable many-body information. For instance, in the case of Coulomb repulsion, successful many-particle configurations will already keep the electrons spatially apart from each other and avoid unfavourable clustering. This suggests a genetic approach which performs random small mutations of currently successful many-body configurations. More precisely, we propose the following 
\\[2mm]
{\bf Genetic search rule.} {\it Given an instance of a reduced constraint matrix $A_I$ and a corresponding RMP solution $\alpha_I$, 
\begin{enumerate}
\item allow only columns $\lambda$ of $A_I$ with $(\alpha_I)_\lambda>0$ to be parents 
\item pick a parent column at random
\item create a child by moving one randomly chosen particle in the parent configuration from its location $a\in X$ to a randomly chosen neighbouring site $a'\in X$.
\end{enumerate}
} 
The last step is crucially based on the physical/geometric meaning of columns as $N$-particle configurations in a region of $d$-dimensional Euclidean space. 

Our rule for creation of children has an interesting {\it metric} meaning in column space which has nothing to do with viewing columns as vectors in $\R^\ell$ and using neighbours with respect to standard distances on $\R^\ell$. Instead, children are obtained from parents by {\it moving a minimum amount of mass by a minimum nonzero Euclidean distance}. To formalize this, let $d\, : \, X\times X\to\R$ be the Euclidean metric $d(x,y)=|x-y|$ on $X=\{a_1,\ldots ,a_\ell\}\subseteq\R^d$ inherited from the ambient $\R^d$. Columns are probability measures on $X$ and for any two columns $\lambda$, $\lambda'\in\calP_{1/N}(X)$ let us introduce their {\it Wasserstein-1 distance} (alias earth-mover's distance) inherited from the ground metric $d$,
$$
    W_1(\lambda,\lambda') = \min\left\{ \int_{X\times X} \!\! d(x,y) d\gamma(x,y) \, | \, \gamma\in\calP(X\! \times\! X), \, \gamma \mbox{ has marginals }\lambda \mbox{ and }\lambda'\right\}
$$ 
Then the rule (3) can be reformulated as: {\it 
\begin{enumerate}
\item[3'.] 
  Pick a random nearest neighbour of the parent in the column space $\calP_{1/N}(X)$ with respect to the Wasserstein-$1$ distance induced by the Euclidean metric on $X\subseteq\R^d$.
\end{enumerate}
} 
We remark that, due to the mass quantization in $\calP_{1/N}(X)$, any of the Wasserstein-$p$ distances with $p\in[1,\infty)$ could be used here instead.

We emphasize that this abstract description of our genetic search rule does not mean that in practice there would be any need to compute Wasserstein distances. In our numerical examples the discretization points are chosen as the intersection of some region $\Omega\subset\R^d$ with a uniform lattice $h\Z^d$ of mesh size $h>0$. One then just needs to pick a random occupied lattice point and updates a random component by $\pm h$. In more sophisticated discretizations like the adaptive one in 3D in \cite{CFM15}, one simply needs to keep a nearest-neighbour list for each discretization point, and make a random choice from this list.

\textcolor{black}{A less stochastic, but slower, variant of 3. would be to generate the best child (in terms of \eqref{eq:PPOne}) among all neighbouring sites $a'$ of the location $a$, or among all children (Wasserstein-1-neighbours) of the parent configuration.}

\subsection{The GenCol algorithm} \label{sec:algo} Based on the results and considerations in the previous sections we propose the following simple algorithm. By an active column we mean a column $\lambda$ for which $(\alpha_I)_\lambda>0$. 
\vspace*{5mm}

\begin{algorithm}[H]
\DontPrintSemicolon
 \KwIn{$N$ (the no. of marginals), $\ell$ (the no. of sites), $\beta$ (hyperpara- meter, chosen to be $5$ in our simulations), $w$ (pair potential), Euclidean coordinates of sites in $\R^d$, marginal $\lambda^*$}
 \KwOut{Solution to the MMOT problem \eqref{eq:ColGenProblemOne}--
\eqref{eq:ColGenProblemThree},\eqref{cost2}, \eqref{A}}
 initialize $A_I$, \, compute $c_I$, \, samples = $0$, \, iter = $0$, \, gain = $-1$ \;
 \While{iter $\le$ maxiter}{
  $\alpha_I$ = solution to RMP \eqref{eq:RMPOne}--\eqref{eq:RMPThree}\;
  $y^*$ = solution to dual problem DRMP \eqref{eq:DRMPOne}--\eqref{eq:DRMPTwo}\;
  \While{gain $\le 0$ and samples $\le$ maxsamples}{
  parent = random active column of $A_I$ \;
  child = new column obtained from parent by randomly moving one particle to a neighbouring site \;
  compute c$_{\rm child}$ (cost of child) using \eqref{cost2} \;
  gain = ${\rm child}^T y_* - c_{\rm child}$ \;
  samples = samples + 1 \;
  }
  $A_I$ = [$A_I$,child], \, $c_I$ = [$c_I$,$c_{\rm child}$] \;
  \If{number of columns of $A_I \,$ $\ge$ $\beta\cdot \ell$}{
      clear oldest $\ell$ inactive columns \;
   }
  iter = iter + 1 \;
 }
 \caption{Genetic column generation ({\tt GenCol})}
 \label{alg:GenCol}
\end{algorithm}
\vspace*{5mm}

The inner while loop generates new columns according to the genetic rule described in section \ref{sec:GenCol} until the acceptance criterion from the pricing problem \eqref{PP'} is satisfied. 

The outer loop is a standard CG iteration in which new columns are added to the current matrix $A_I$ of the restricted master problem (RMP) and the primal and dual solutions are updated. 

To prevent the size of $A_I$ from growing too large, the ``oldest'' inactive columns are cleared whenever a maximum allowed size has been reached. The maximum size is defined with the help of the hyperparameter $\beta$; we do not allow to exceed the minimum size for exactness of the method (namely $\ell$, see Theorem \ref{T:sparse}) by more than a factor $\beta$. The meaning of ``oldest'' is oldest with respect to having been found; the empirical rationale here is that older columns were found with the help of a less accurate dual solution. 
\section{NP-completeness of the pricing problem} \label{sec:NP}
The formula for fast cost evaluation derived in Corollary \ref{C:fastcost} means that the pricing problem \eqref{eq:PPOne}--\eqref{eq:PPTwo} for MMOT is a linearly constrained integer-optimization problem with quadratic objective:
\begin{align}
\label{eq:PPIntOne}
\textrm{Maximize }		& \lambda^Ty^*-\frac{N}{2}\lambda^TC\lambda+\frac{N}{2}\textrm{diag}(C)^T\lambda \\
\label{eq:PPIntTwo}
\textrm{subject to }	& \sum_{i=1}^\ell \lambda_i=N\\
\label{eq:PPIntThree}
                     	& \lambda\in \mathbb{N}^{\ell}_0. 
\end{align}
To derive this form, we have rescaled the objective function and the computational domain by a factor $N$. 

As the theory of NP-completeness evolves around decision problems, we start by formulating a 'decision version' of the pricing problem \eqref{eq:PPIntOne}-\eqref{eq:PPIntThree}.

The objective function of the pricing problem consists of a quadratic and a linear term. The linear term depends on the cost matrix for the quadratic objective. Nevertheless we formulate the \textit{Pricing Decision Problem} by treating both terms as independent. 
\\ \\
\noindent\makebox[\textwidth][c]{
\begin{minipage}{0.8\textwidth}
\textit{PDP: Given natural numbers $N,\ell \in \mathbb{N}$, a cost matrix $V\in \mathbb{R}^{\ell\times\ell}$, a cost vector $a\in \mathbb{R}^{\ell}$},  
\textit{and a threshold $K\in\mathbb{R}$, does there exist a vector $\lambda\in\{0,1,\ldots,N\}^{\ell}$ such that}
\begin{equation*}
\sum_{i=1}^\ell \lambda_i=N	
\end{equation*}
 \textit{and}
\begin{equation*}
\lambda^TV\lambda+a^T\lambda\geq K?	
\end{equation*}
\end{minipage}}
\vspace*{3mm}

Even if we restricted our attention to the choices of input parameters covered by our pricing problem \eqref{eq:PPIntOne}-\eqref{eq:PPIntThree}, we still would be able to establish the NP-completeness of the PDP; see Remark \ref{rem:NPcompleteness} below. 

The remainder of this section is devoted to proving NP-completeness of the PDP. This result strongly calls into question the possibility of a polynomial time algorithm for the PDP. In fact it also calls into question the possibility of a polynomial time algorithm that solves the PP \eqref{eq:PPIntOne}-\eqref{eq:PPIntThree}, by the following argument. Suppose that such an algorithm exists. 
Given an instance of the PDP one is now able to compute the optimal value of the corresponding PP and simply compare it to the threshold of the given instance. So also the PDP would be solvable in polynomial time. 

To prove NP-completeness of the PDP we will use the following elementary lemma, whose proof is included for completeness.

\begin{lem}
\label{lem:NPcompleteness}
Given a natural number $q$, let $E^q\in\mathbb{R}^{q\times q}$ be the matrix whose diagonal entries are equal to zero whereas all off-diagonal entries are equal to one. Then 
\begin{equation}
\label{eq:LambdaTilde}
\tilde{\lambda} \textrm{ with } \tilde{\lambda_i}=1 \textrm{ for all }i\in\{1,\ldots,q\}	
\end{equation}
is the unique maximizer of the problem 
\begin{align*}
\textrm{Maximize }		& \lambda^TE^q\lambda \\
\textrm{subject to }	& \sum_{i=1}^q \lambda_i=q\\
                     	& \lambda\in \mathbb{R}^{q} 
\end{align*}
whose optimal value is therefore given by $q(q-1)$.
\end{lem}
\begin{proof}
The matrix $E^q$ has the two eigenvalues $q-1$ and $-1$. Corresponding eigenvectors are given by $v_1=\tilde{\lambda}$ as well as $v_2=e_2-e_1, v_3=e_3-e_1, \ldots, v_{q}=e_q-e_1$. Here $e_i$ denotes the $i$-th standard unit vector in $\R^d$. Geometrically, the eigenvector $v_1$ takes us onto the hyperplane our optimization problem is 'living' on. Any further movement corresponds to an addition of a linear combination of the eigenvectors $v_2, \ldots, v_q$ with the negative eigenvalue $-1$, and thereby a decrease of the objective value. In formulas, let us write any admissible trial state $\lambda \in \mathbb{R}^q$ as a linear combination of eigenvectors, 
$\lambda=\sum_{i=1}^q	\alpha_iv_i \mbox{ with }\alpha_1,\ldots ,\alpha_q\in\R$. Multiplication with $(1,\ldots,1)$ immediately shows that -- by admissibility of $\lambda$ -- $\alpha_1=1$ and therefore
\begin{equation*}
\lambda=v_1+\sum_{i=2}^q	\alpha_iv_i.
\end{equation*}
As $v_1$ is perpendicular to $v_2,\ldots,v_{q}$,  
$$
\lambda^T E^q \lambda =v_1^TE^qv_1-\left|\sum_{i=2}^q\alpha_iv_i\right|^2 
				      \begin{cases}
	                     =q(q-1)  &\textrm{if } \alpha_2=\cdots=\alpha_q=0 \\
	                     < q(q-1) & \textrm{else}.\end{cases}	
$$
This establishes the optimality of $v_1=\lambda$.  	
\end{proof}
The main result of this section is:
\begin{thm}
\label{thm:NPC}
The PDP is NP-complete.	
\end{thm}
Before we come to the proof, let us recall what it is that needs to be proven. As discussed for example in \cite{KPP04, CLRS09}, a decision problem Q is classified as NP-complete if it is (i) contained in the class NP and (ii) 'at least as hard' as any other problem in NP. \textcolor{black}{This} class  consists of those decision problems for which 'yes'-instances can be verified in polynomial time. Regarding (ii) we will use the concept of polynomial reduceability. A decision problem S is said to reduce (or transform) to another decision problem T in polynomial time, if there exists a polynomial time function f that maps any instance $I$ of S to an instances $f(I)$ of T in such a manner that $I$ is a 'yes'-instance of S if and only if $f(I)$ is a 'yes'-instance of T. Then as polynomial solvability of T implies polynomial solvability of S, T is considered 'harder' as S. Consequently, in order to prove that Q is NP-complete one needs to show that any problem in NP can be polynomially reduced to Q. As a result of the transitivity of polynomial reduceability, one can show this by picking a known NP-complete problem and proving that it reduces to Q. In the following we will establish the NP-completeness of the PDP, i.e., prove Theorem \ref{thm:NPC}, using this common approach. We will show that the following \textit{Clique Decision Problem} (CDP) reduces to the PDP in polynomial time:
\\[2mm]
\makebox[\textwidth][c]{
\begin{minipage}{0.8\textwidth}
\textit{CDP: Let $G:=(V,E)$ be an undirected graph and $K' \in \mathbb{N}$ be a natural number that fulfils $K'\leq|V|$. Does there exist a clique of size at least $K'$?} 
\end{minipage}}
\\[2mm]
Recall that given an undirected graph $G:=(V,E)$ a clique $C$ corresponds to a subset of the vertices of G, i.e., $C\subseteq V$, such that every distinct pair of vertices $c_1,c_2\in C$ is connected by an edge, i.e., $\{c_1,c_2\}\in E$. We refer the interested reader to \cite{CLRS09} for a more detailed discussion of the CDP including a proof of its NP-completeness. 

\begin{proof}
First, it is easy to see that the PDP is contained in the class NP. Assume that we are given a capacity $N$, a size parameter $\ell$, a threshold $K$, a cost matrix $V$, and a cost vector $a$ that yield the answer 'yes' if used as input arguments of the PDP. Now let $\lambda$ be one of those vectors about whose existence the PDP asks. Then $\lambda$ corresponds to a certificate, whose size is polynomial in the size of the input, for which the capacity as well as the threshold constraint can be checked in polynomial time. Thus overall the 'yes'-instance can be verified in polynomial time.   

Next we prove that the CDP can be reduced to the PDP in polynomial time. A given instance $I=(G,K')$ of the CDP is hereby mapped to an instance $f(I)=(N,\ell,K,a,V)$ of the PDP where $N$ is set to $K'$, $\ell$ corresponds to the cardinality $|V|$ of the vertex set of G, the threshold $K$ is given by $K'(K'-1)$, the cost vector $a$ equals the zero vector and finally the cost matrix $V$ is set to be the adjacency matrix $A_G=(a_{ij})_{i,j=1}^{\ell}$ of the graph $G$ which fulfils
\begin{equation*}
a_{ij}=\begin{cases}
1 & \textrm{if } \{i,j\}\in E\\
0 & \textrm{else.}	
\end{cases}	
\end{equation*}
 Thereby the adjacency matrix is, as usual, based on a given order of the vertices. It is easy to see that for a given instance $I$ the matching instance $f(I)$ can be computed in polynomial time. What remains to be shown is that for the described mapping $I$ is a 'yes'-instance regarding the CDP if and only if $f(I)$ is a 'yes'-instance with respect to the PDP.

Let $I=(G,K')$ be such a 'yes'-instance regarding the CDP. Then $G=(V,E)$ contains a clique $C\subseteq V$ of size $K'$. Assume further $\lambda$ to be the vector indicating which vertices are contained in $C$, i.e., $\lambda\in \{0,1\}^{\ell}$ with  
\begin{equation*}
\lambda_i=\begin{cases}
1 & \textrm{if } i \in C\\
0 & \textrm{else.}	
\end{cases}
\end{equation*}
Then it is elementary to check that the entries of $\lambda$ sum to $K'$ and that $\lambda^TA_G\lambda=K'(K'-1)$. Consequently, the given $\lambda$ triggers a 'yes'-answer of the PDP with respect to the instance $f(I)$. 

Now assume $f(I)$ to be a 'yes'-instance with respect to the PDP. Then there exists a vector $\lambda\in \{0,1,\ldots,K'\}^\ell$ such that its entries sum to $K'$ and it satisfies $\lambda^TA_G\lambda\geq K'(K'-1)$. We will show that this vector $\lambda$ only consists of zero- and one-entries where the latter indicate a set of vertices that forms a clique in the graph G of size $K'$. This immediately determines that $I$ is a 'yes'-instance regarding the CDP. 

In the following, $I\subseteq\{1,2,\ldots,\ell\}$ will denote a set of indices that satisfies 
\begin{equation*}
\lambda_i>0 \textrm{ implies }i\in I	
\end{equation*}
as well as
\begin{equation*}
|I|=K'.	
\end{equation*}
One can generate $I$ by filling up the set of 'support-indices' of $\lambda$ with an arbitrary choice of the remaining indices in $\{1,2,\ldots,\ell\}$. This is possible as by assumption $K'\leq \ell$. The $K'$-dimensional vector consisting only of those entries of $\lambda$ that correspond to indices in $I$ will in the following be denoted by $\lambda_I$. Accordingly, $A_G^I=(a_{ij})_{i,j\in I}$ denotes the matrix that is built from only those rows and columns of $A_G$ with indices in $I$.
\\
By the assumption on $\lambda$ and the construction of $I$ we have
\begin{equation}
\label{eq:NPcompletenessProofOne}
K'(K'-1)\leq \lambda^TA_G\lambda= \lambda_I^TA_G^I\lambda_I.
\end{equation}
Since replacing $A_G^I$ with $E^{K'}$ as defined in Lemma \ref{lem:NPcompleteness} does not decrease the value of the quadratic function for nonnegative input arguments, 
\begin{equation}
\label{eq:NPcompletenessProofTwo}
\lambda_I^TA_G^I\lambda_I \leq \lambda_I^TE^{K'}\lambda_I.
\end{equation}
Combining \eqref{eq:NPcompletenessProofOne} and \eqref{eq:NPcompletenessProofTwo} and applying Lemma \ref{lem:NPcompleteness} for $q=K'$ yields
\begin{equation}
\label{eq:NPcompletenessProofThree}
K'(K'-1)\leq \lambda^TA_G\lambda=\lambda_I^TA_G^I\lambda_I\leq\lambda_I^TE^{K'}\lambda_I\leq K'(K'-1).	
\end{equation}
Every inequality in \eqref{eq:NPcompletenessProofThree} is now actually an equality; consequently $\lambda_I$ equals the vector $\tilde{\lambda}$ from \eqref{eq:LambdaTilde} with $q=K'$, and therefore we also have $A_G^I=E^{K'}$. Since $A_G^I$ corresponds to the adjacency matrix of the subgraph of G induced by $I$, G indeed contains a clique of size $K'$, namely $I$. This concludes the proof of Theorem \ref{thm:NPC}. 
\end{proof}
\begin{rem}
\label{rem:NPcompleteness}
We claim that the PDP is still NP-complete if the cost matrix $V$ and the cost vector $a$ are restricted to be of the special form in \eqref{eq:PPIntOne}-\eqref{eq:PPIntThree}. In order for the PDP to remain NP-complete there need to exist reduced problems - RMP, DRMP - that underlie the choices of $A_G$ and the zero vector as cost matrix $V$ and cost vector $a$. Define $I$ in such a manner that $A_I$ corresponds to the $\ell \times \ell$-identity matrix, and let $C=-\frac{2}{N}A_G$. Then the objective of the pricing problem \eqref{eq:PPIntOne}-\eqref{eq:PPIntThree} takes on the required form. Consequently Theorem \ref{thm:NPC} remains valid if one restricts the choices of cost matrix and cost vector to the ones arising in \eqref{eq:PPIntOne}. 
\end{rem}

%
%
%
%
%

\section{Numerical Results}
\label{sec:numerical}
All our tests were performed for our key motivating application, the Coulomb problem \eqref{MMOT}--\eqref{CoulombCase}. For simplicity we used the regularized Coulomb interaction $w(x,y)=1/\sqrt{\eps^2 + |x-y|^2}$ with $\eps=0.1$.

\subsection{Ten electrons in 1 D with inhomogeneous density}
As a first test we ran the {\tt GenCol} algorithm on problem \eqref{MMOT}--\eqref{CoulombCase} with ten electrons in a 1D interval discretized by $\ell=100$ uniformly spaced gridpoints, for the marginal density shown in Figure \ref{F:density}. We normalized the spacing to $1$ and took the density as a function of the gridpoints $a_i=i\in\{1,\ldots ,\ell\}\subset\R$ to be $\lambda^*(\{a_i\}) = \textcolor{black}{c_0} \cdot ( 0.2 + \sin^2\bigl(\frac{i}{\ell+1}\bigr) )$, \textcolor{black}{where $c_0$ is a normalization constant so that $\sum_i \lambda^*(\{ a_i \} )=1$}. 

\begin{figure}[h]
\begin{center}
\includegraphics[width=0.42\textwidth ]{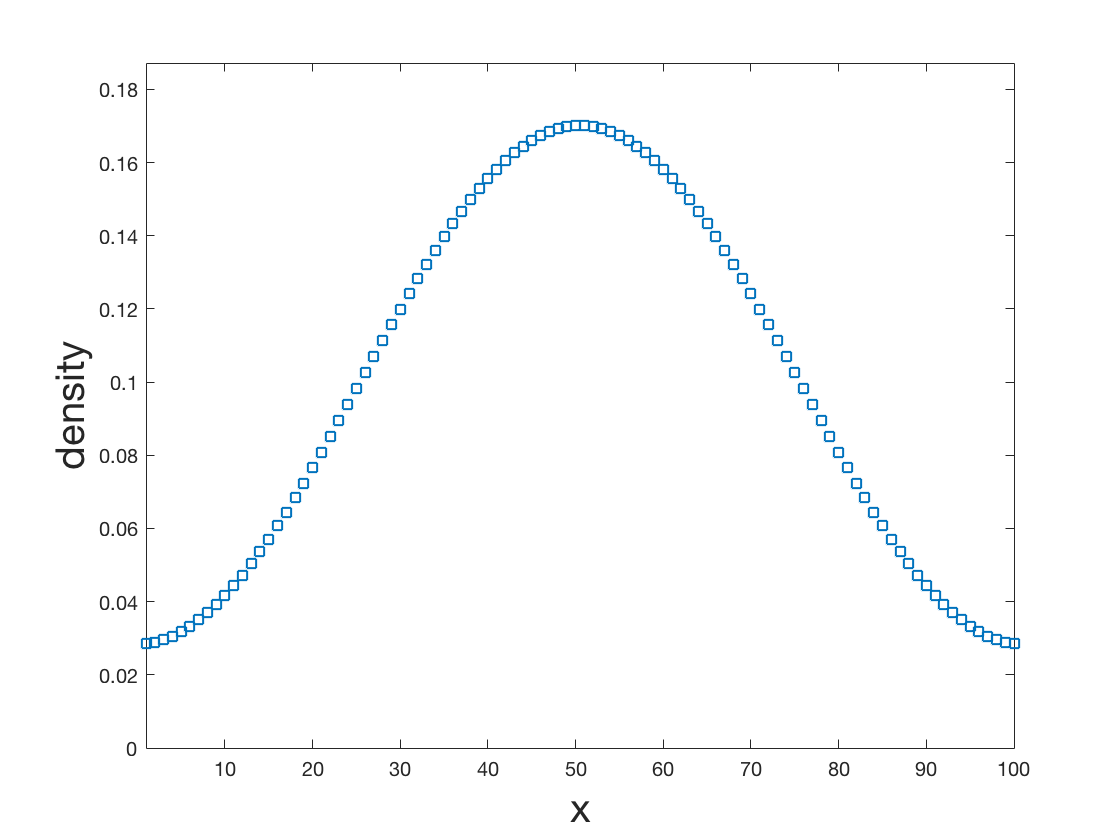}
\end{center}
\vspace*{-4mm}

\caption{Prescribed single-particle density}
\label{F:density}
\end{figure}

We initialized the matrix $A_I$ with the $\ell$ columns of the $\ell\times\ell$ identity matrix (to ensure that the optimization in the RMP \eqref{eq:RMPOne}--\eqref{eq:RMPThree} is feasible) as well as $(\beta - 1)\times \ell$ random columns, each of them obtained by dropping $N$ particles randomly with respect to the uniform measure onto the grid. The results are given in Figure \ref{F:1Di}. After less than $7~000$ iterations the algorithm found what we believe to be the exact solution (within machine precision). Due to the problem size of $4.26\times 10^{13}$ possible columns, rigorous certification of the solution is out of the question, but we tested it both by a long (and, as turned out, futile) non-genetically-biased search for better columns and by re-running the simulation many times, always ending up with the same state. 

The multi-marginal Kantorovich plan (or $N$-point density), visualized via its two-point marginal (or pair density), is seen to concentrate on the graphs of $N-1=9$ maps, thereby accurately reproducing the known behaviour of the continuous problem as predicted by Seidl \cite{Se99} and rigorously proved in \cite{CDD15}. Overall only about 33000 columns out of the $4.26\times 10^{13}$ possible columns were sampled in order to find the ground state solution. The cost decreased steadily at an exponential rate (see Figure \ref{F:1Di_performance}).

\begin{figure}[http!]
\includegraphics[width=0.15\textwidth ]{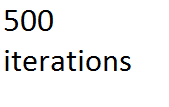}
\hspace*{-0.6cm}
\includegraphics[width=0.45\textwidth ]{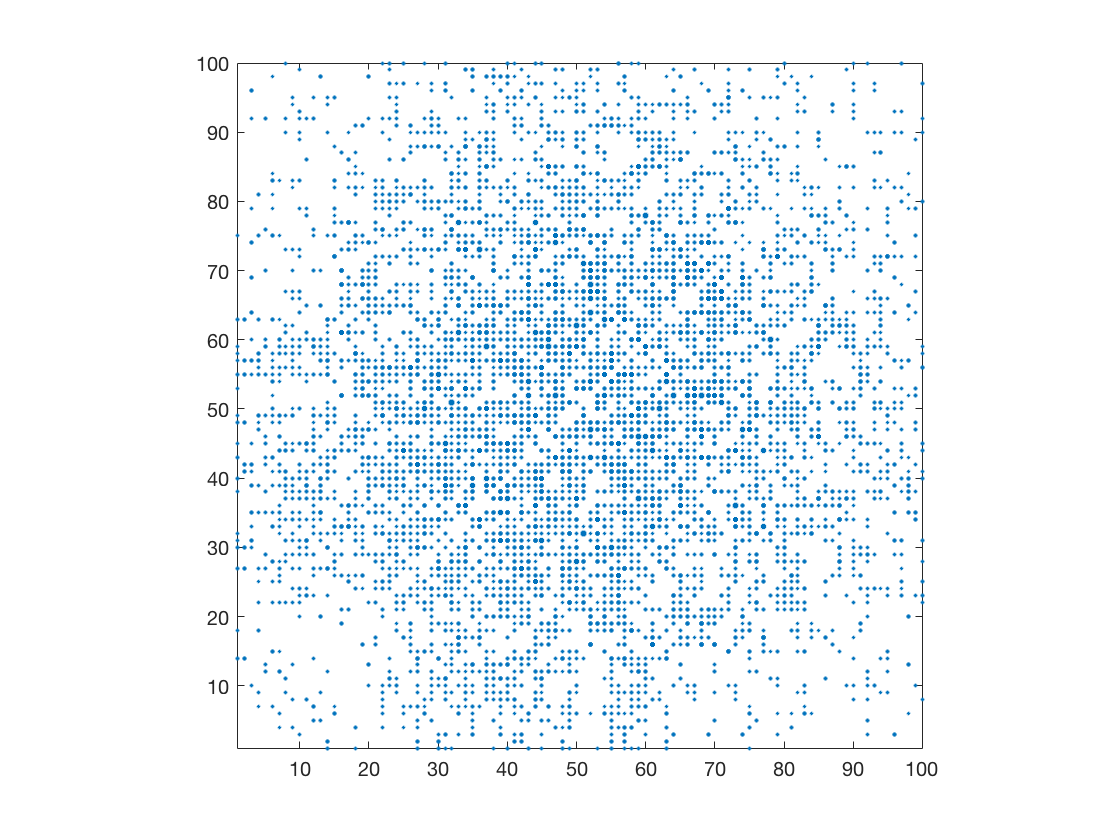} 
\includegraphics[width=0.38\textwidth ]{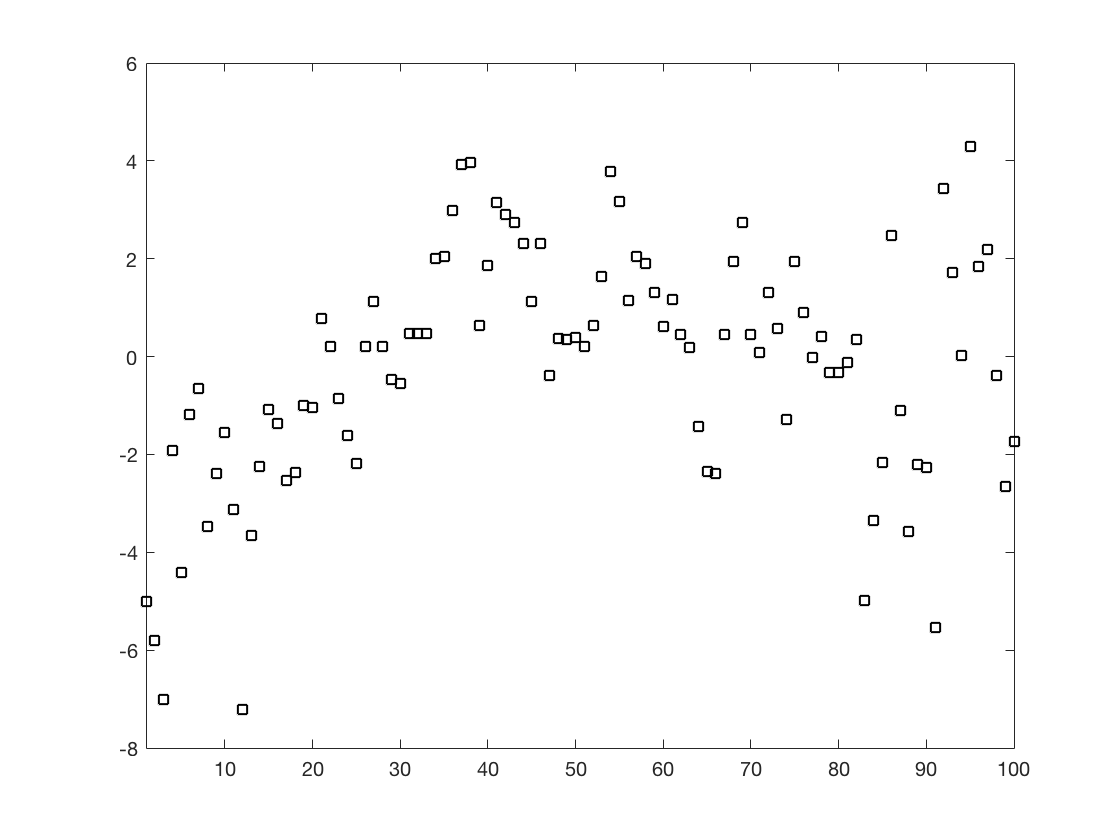} \\
\includegraphics[width=0.15\textwidth ]{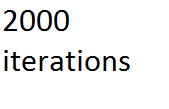}
\hspace*{-0.6cm}
\includegraphics[width=0.45\textwidth ]{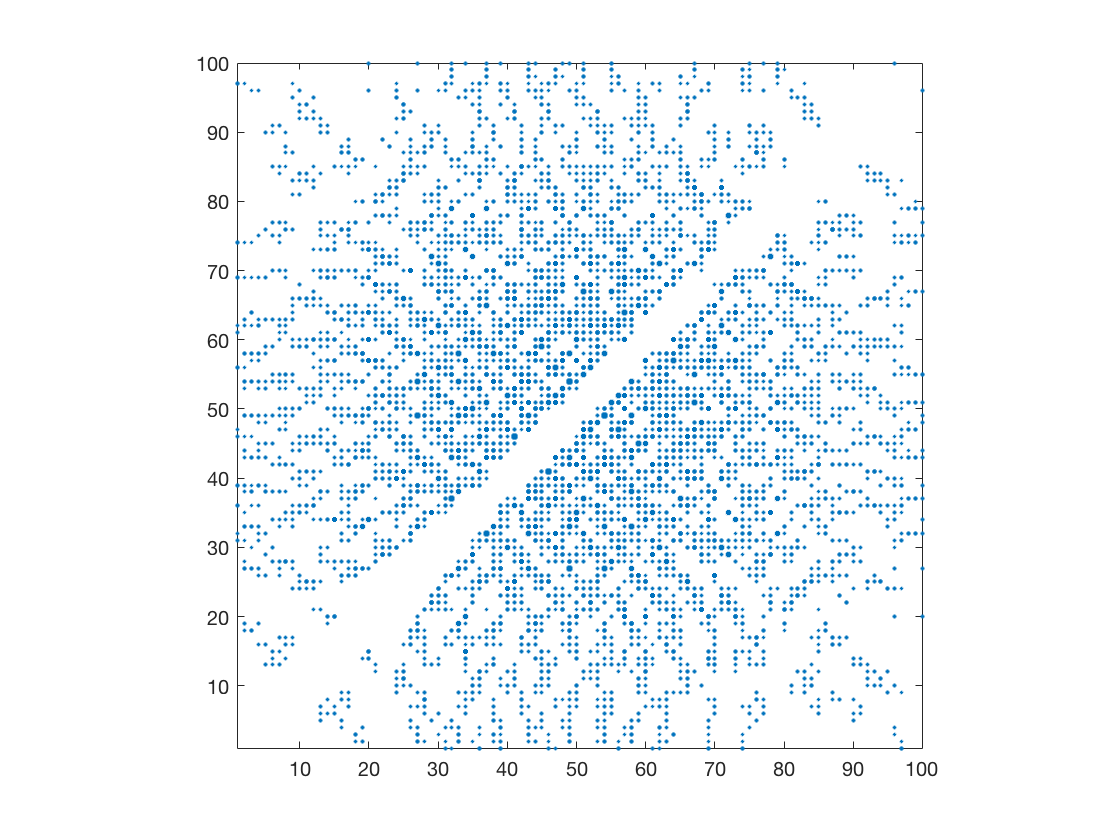} 
\includegraphics[width=0.38\textwidth ]{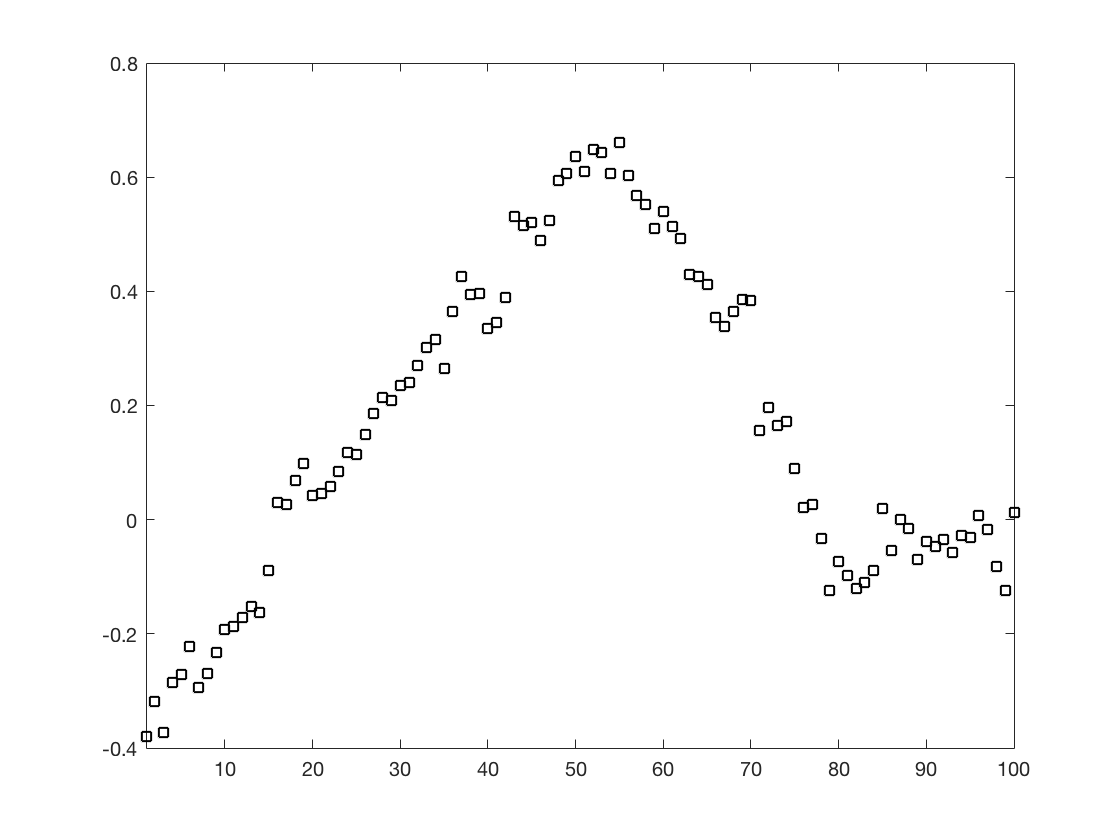} \\
\includegraphics[width=0.15\textwidth ]{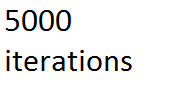}
\hspace*{-0.6cm}
\includegraphics[width=0.45\textwidth ]{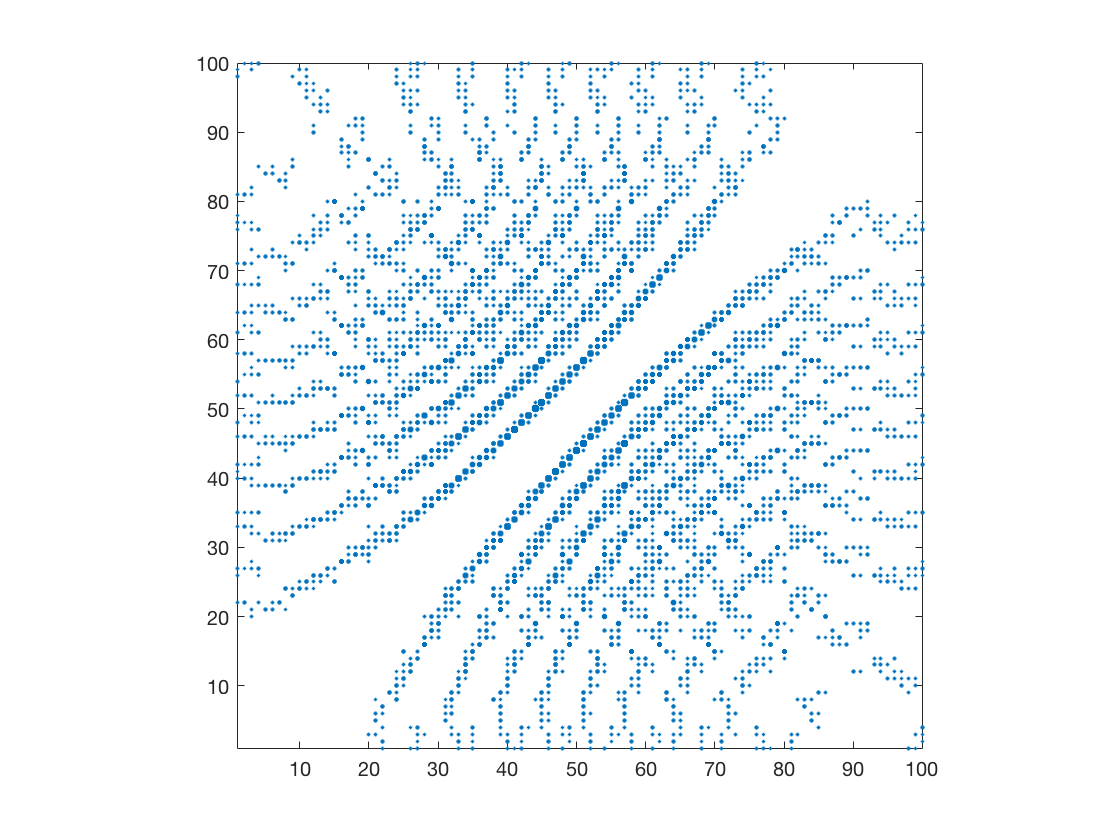} 
\includegraphics[width=0.38\textwidth ]{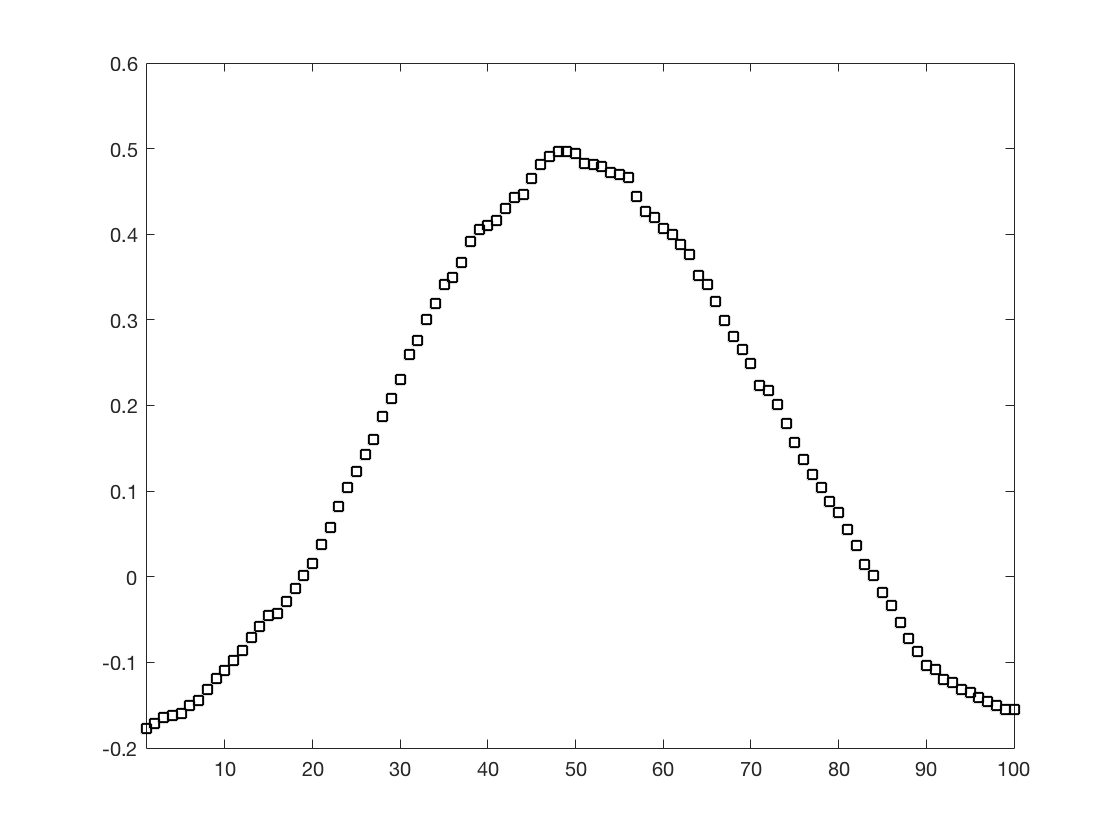} \\
\includegraphics[width=0.15\textwidth ]{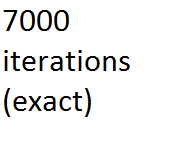}
\hspace*{-0.6cm}
\includegraphics[width=0.45\textwidth ]{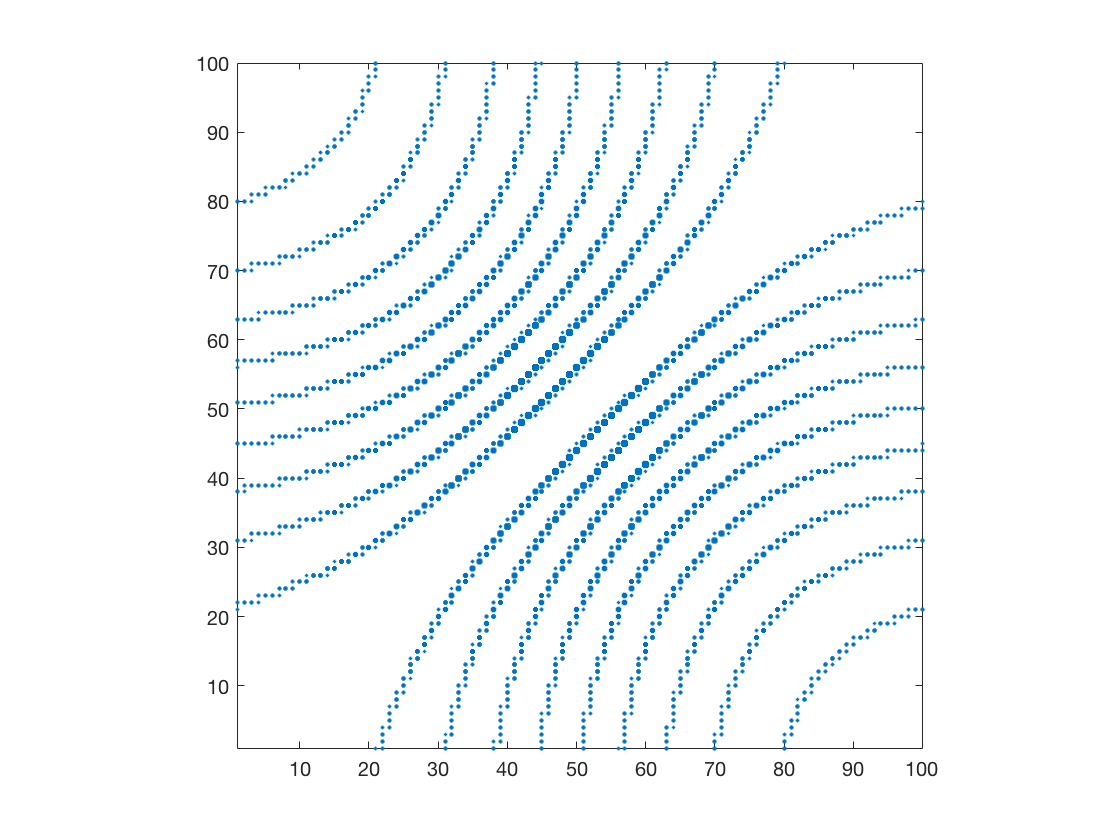} 
\includegraphics[width=0.38\textwidth ]{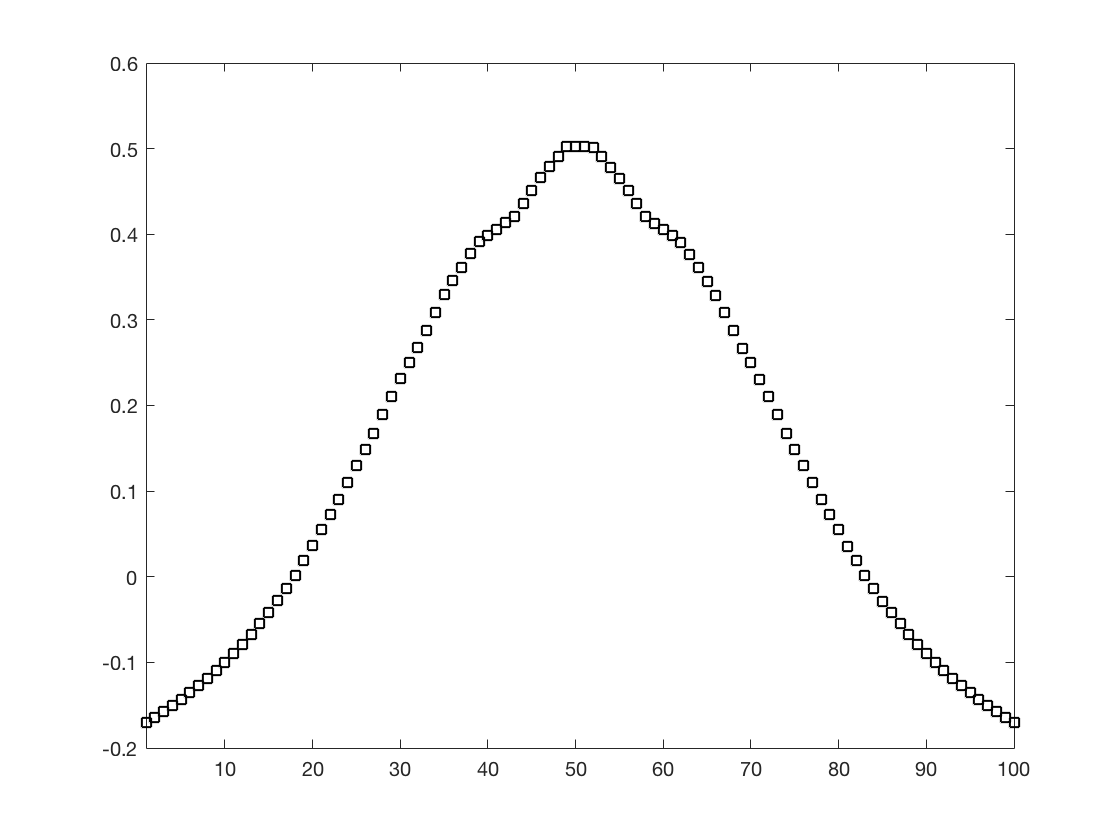}

\caption{Solution to multi-marginal optimal transport with Coulomb cost for $10$ electrons in 1D with the {\tt GenCol} algorithm. The prescribed one-point marginal (single-electron density) is depicted in Figure \ref{F:density}, and was discretized by $100$ gridpoints, resulting in $4.26\times 10^{13}$ unknowns (or 'columns') in the full linear program. {\it Left:} Evolution of the multi-marginal plan, visualized via its two-point marginal (pair density). All gridpoints with nonzero values are shown in blue, with larger markers indicating higher values. {\it Right:} Evolution of the dual solution (Kantorovich potential). The final plan -- believed to be the exact ground state within machine precision -- was found using 6789 iterations (accepted columns) and 33283 samples (genetically generated columns).}
\label{F:1Di}
\end{figure}

\begin{figure} 
\begin{center}
\includegraphics[width = 0.55\textwidth ]{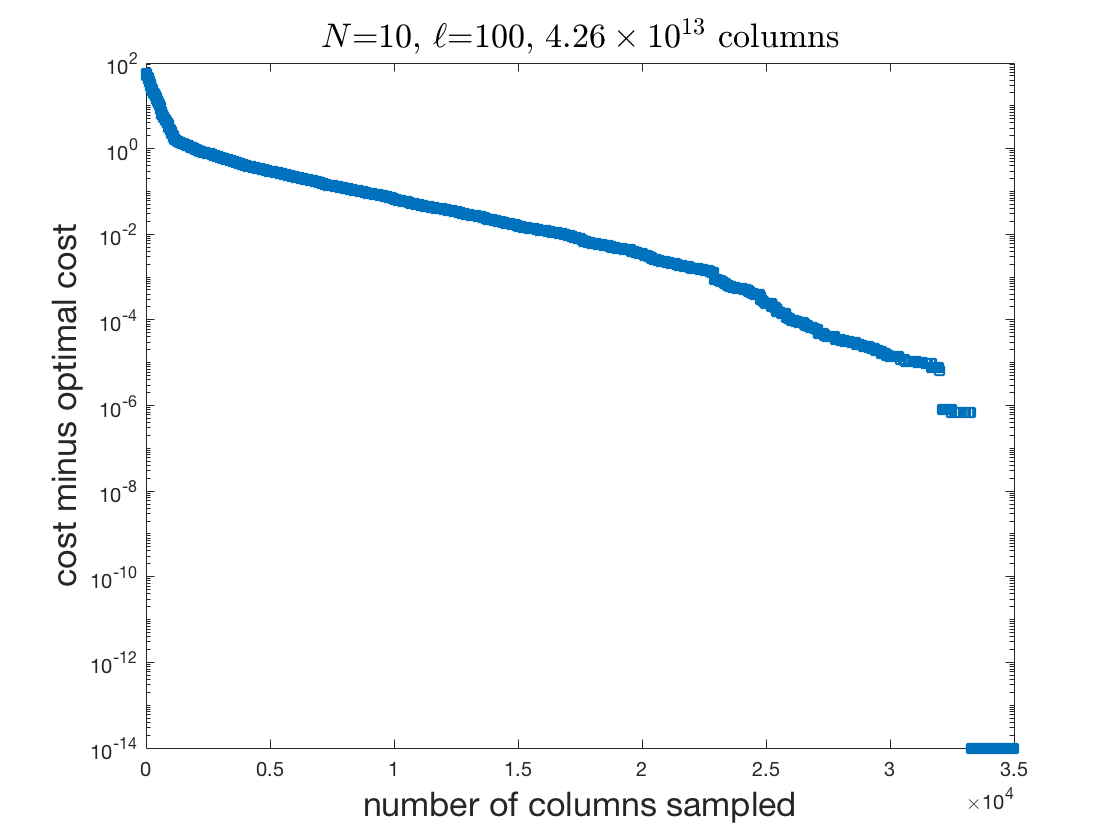}
\end{center}
\caption{Evolution of the cost of the RMP solution, for the $10$-electron simulation from Figure \ref{F:1Di}}  
\label{F:1Di_performance}
\end{figure}

From an unsupervised learning perspective, the Kantorovich potential plays the same role in the {\tt GenCol} algorithm for MMOT as it does in the W-GAN algorithm \cite{ACB17} for learning unknown distributions from data, namely that of an ``adversary''. In the initial stages the adversary is not of much help (it looks close to a random potential) and the primal state has difficulty learning anything other than the -- physically obvious -- fact that two electrons being extremely close is costly. As the number of iterations increases, primal and dual state steadily acquire finer and finer characteristics until reaching optimality. We attribute the success of the {\tt GenCol} algorithm in overcoming the vastness of the space of possible Kantorovich plans to the ability of primal state and dual state to ``learn from each other''. 
\subsection{Large N-electron systems in 1D; cost scaling}
We now empirically investigate the important issue of how the computational cost of the {\tt GenCol} algorithm scales with system size. As a suite of test systems we choose MMOT with Coulomb cost in 1D and homogeneous marginal $\lambda^*$, with an increasing number $N$ of electrons and an increasing number $\ell$ of gridpoints. In fact, it is physically natural to {\it increase both parameters simultaneously} and consider a sequence of systems with
\begin{equation} \label{thermo}
   \mbox{increasing }N, \;\;\; \mbox{increasing }\ell, \;\;\; \frac{N}{\ell}\equiv const.
\end{equation}
In the limit $N\to\infty$, $\ell\to\infty$, $\frac{N}{\ell}\equiv const$ (so-called thermodynamic limit) the system approaches the 1D homogeneous electron gas. At fixed mesh size (normalized to $1$ in our simulations), the condition $\frac{N}{\ell}\equiv const$ means physically that we increase the available volume proportionally to the number of particles, thereby allowing typical interparticle distances to stay unaltered, as happens in large molecules and solids in nature. 

The above family of systems has the advantage that for integer values of $\frac{N}{\ell}$ the exact solution to \eqref{one''}--\eqref{three''} is known even after discretization (or, more precisely, it can be deduced via the same methods with which the exact solution for the continuous theory has been derived in \cite{CDD15}). It consists of the symmetrized Monge state
\begin{equation} \label{thermosol}
   \gamma_{i_1,\ldots ,i_N} = S_N \sum_{i_1=1}^\ell \lambda_{i_1}^* \prod_{k=2}^N \delta_{i_k,i_1 + (k-1)\tfrac{\ell}{N}}
\end{equation} 
which represents a superposition of uniformly spaced $N$-particle configurations. Here  $\delta_{i,j}$ denotes the Kronecker delta function.

We ran the {\tt GenCol} algorithm on the sequence of systems
\begin{equation} \label{eq:thermo}
  \left\{ \!\!\begin{array}{l} N=5 \\  \ell = 20 \end{array} \right. \! , \;\;
  \left\{ \!\!\begin{array}{l} N=10 \\ \ell = 40 \end{array} \right.\! , \;\;
  \left\{ \!\!\begin{array}{l} N=15 \\ \ell = 60 \end{array} \right.\! , \;\;
  \left\{ \!\!\begin{array}{l} N=20 \\ \ell = 80 \end{array} \right. \! , \;\;
  \left\{ \!\!\begin{array}{l} N=25 \\ \ell = 100 \end{array} \right.\! , \;\;
  \left\{ \!\!\begin{array}{l} N=30 \\ \ell = 120 \end{array} \right.\! ,
\end{equation}
with $5$ different runs for each system. We initialized the matrix $A_I$ with the $\ell$ columns of the identity matrix (for feasibility), augmented by $N\cdot \ell$ random columns. In every single case {\tt GenCol} found the exact solution. See Figure \ref{F:1Dh} for the evolution of the Kantorovich plan for $N=25$, $\ell=100$. The number of iterations and genetic samples needed to find the exact solution are given in Table 1.  

\begin{figure}[http!]
\begin{center}
\includegraphics[width=0.48\textwidth ]{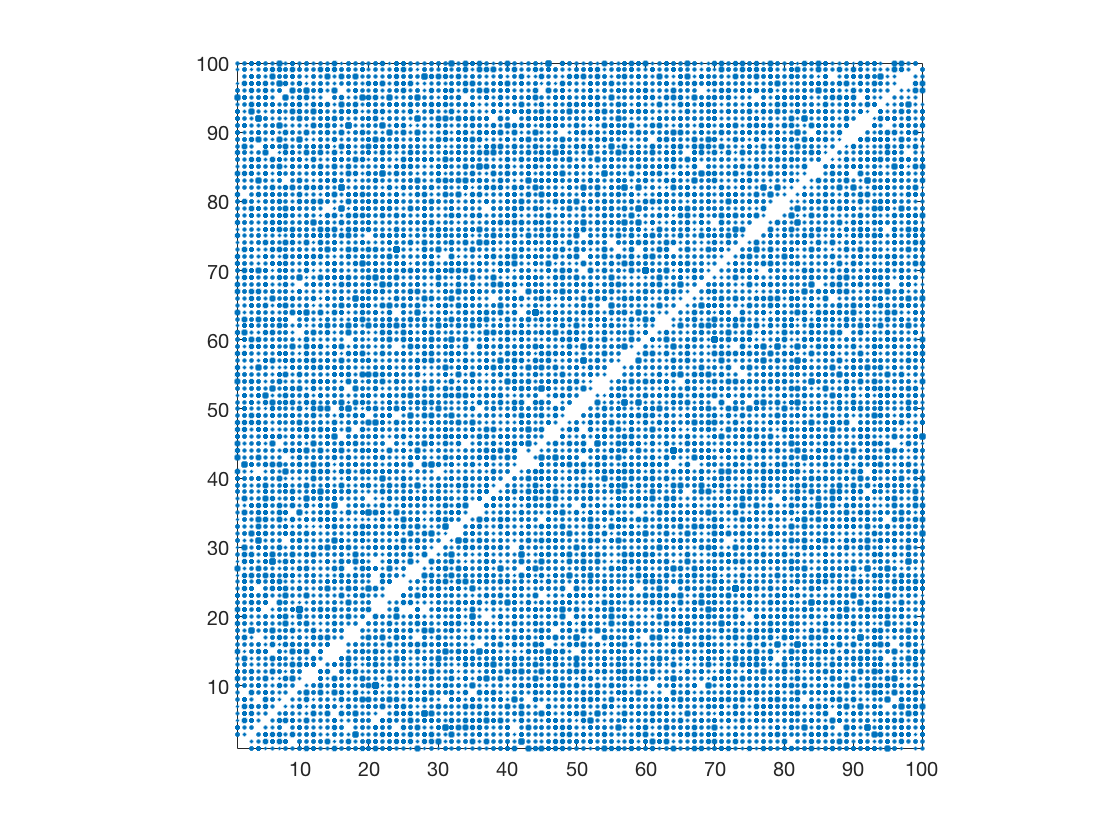} \hspace*{-1cm}
\includegraphics[width=0.48\textwidth ]{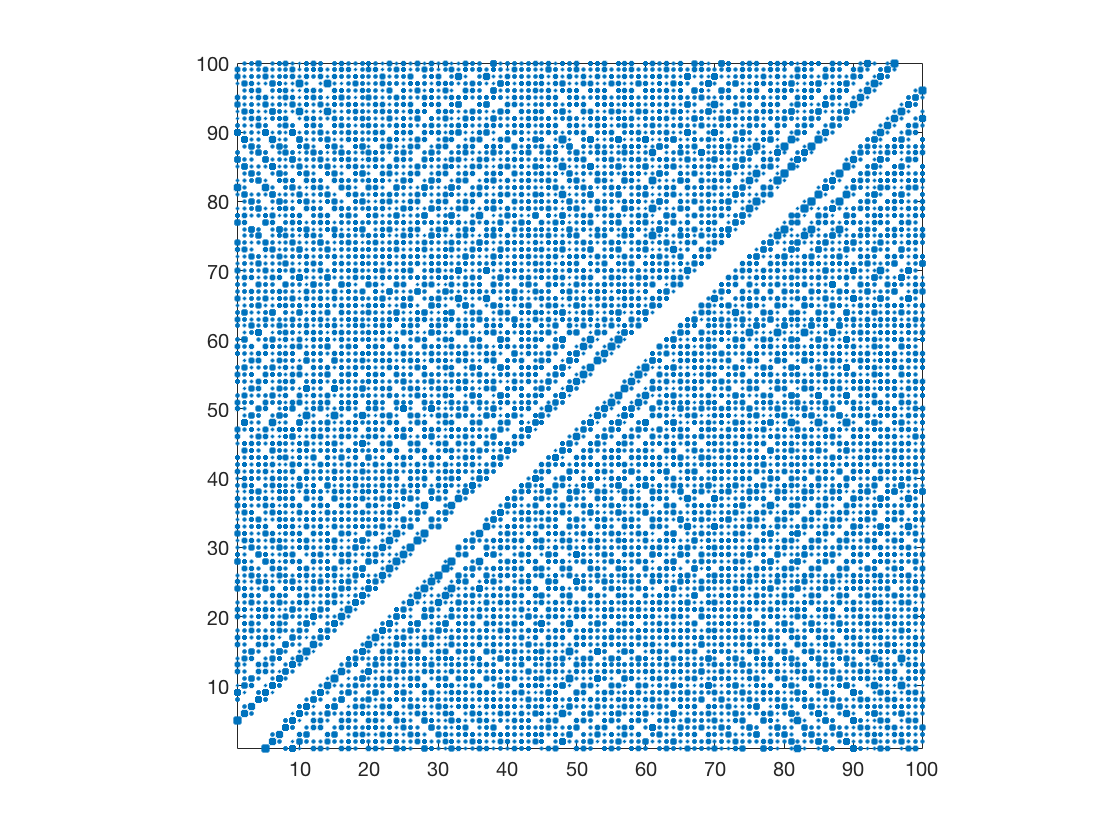} \\
\includegraphics[width=0.48\textwidth ]{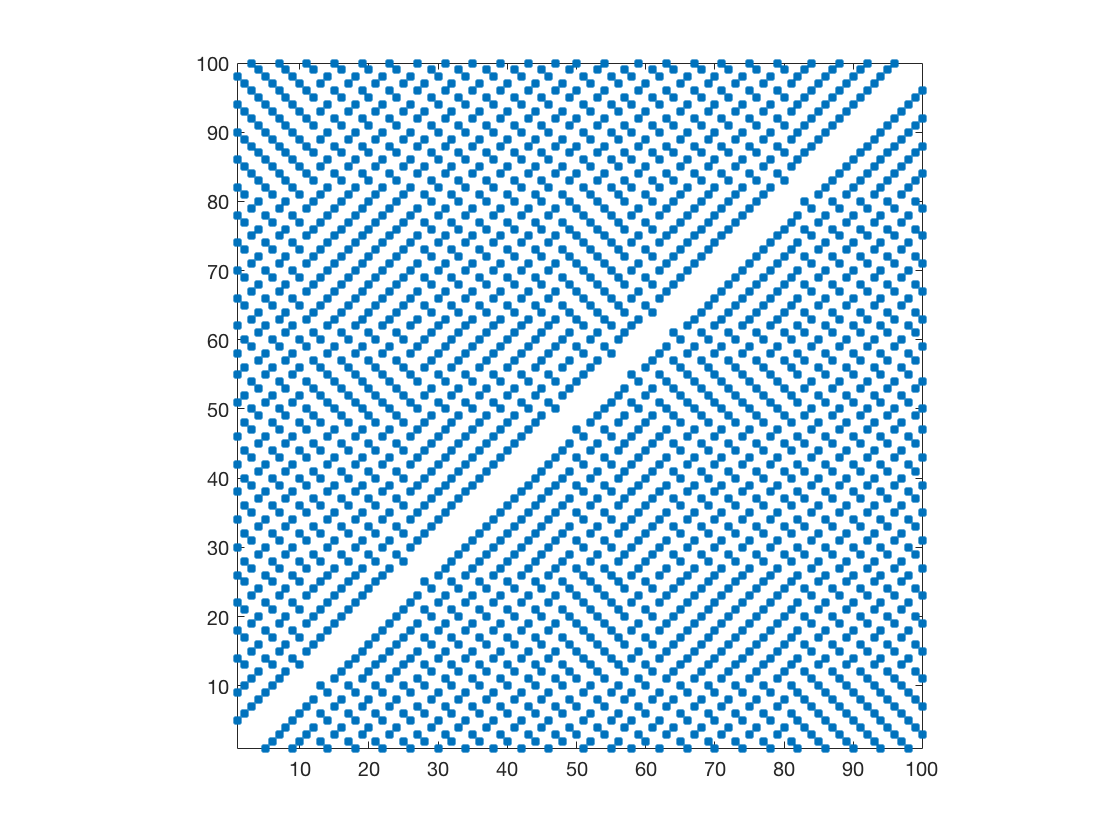} \hspace*{-1cm}
\includegraphics[width=0.48\textwidth ]{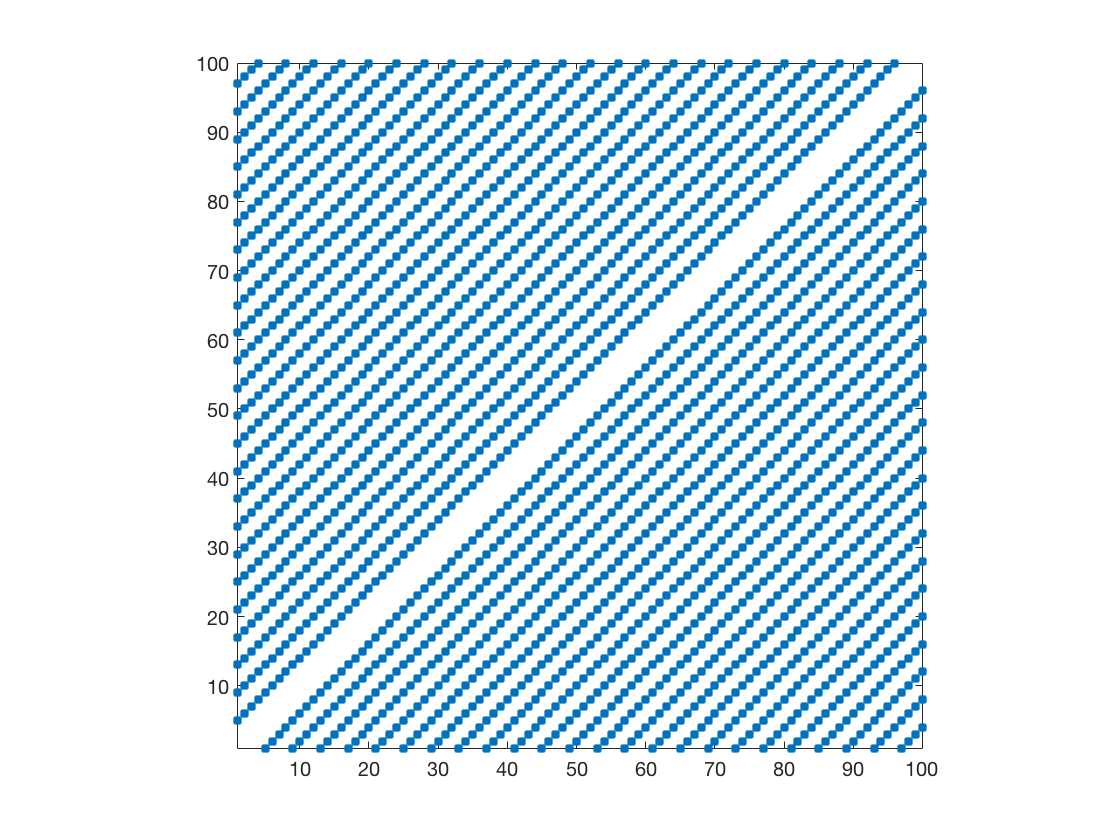} 
\end{center}
\vspace*{-4mm}

\caption{Solution to multi-marginal optimal transport with Coulomb cost for $25$ electrons in 1D with the {\tt GenCol} algorithm, with prescribed homogeneous marginal (single-electron density). The marginal was discretized by $100$ gridpoints, resulting in $1.0404 \times 10^{26}$ unknowns (or 'columns') in the full linear program. Multi-marginal plans are visualized via their two-point marginal (pair density) as described in Figure \ref{F:1Di}. {\it Top left to bottom left:} 1000, 8000, 9000 iterations. {\it Bottom right:} exact solution, reached after 9322 iterations (accepted columns) and 38860 samples (genetically generated columns).}
\label{F:1Dh}
\end{figure}

Since each iteration only involves solving a linear program for at most $\beta\cdot\ell$ unknowns and $\ell$ constraints (where $\beta=5$ in our case), and we limited the number of iterations in the linear programming solver used (Matlab's {\tt linprog}) to $O(\ell^2)$, the key limiting factor is the number of genetic samples needed. Figure 5 shows a log-log-plot of the average number of genetic samples needed for each system. While the system size (i.e., the number of unknowns) grows exponentially, the number of genetic samples needed to find the exact solution appears to lie on a straight line, suggesting polynomial growth only. This is particularly remarkable in the light of our result in section \ref{sec:NP} that the pricing problem -- which our genetic sampling method addresses -- is NP-complete.  

\begin{table}
\begin{tiny} 
\begin{center}
\begin{tabular}{|l|p{20mm}|p{25mm}|p{25mm}|p{25mm}|} 
\hline
System  & total number \newline of columns & accepted \newline columns & sampled \newline columns & sampled columns \newline (average) \\[1mm]
\hline 
$N=5$,             & $4.2504 \times 10^4$  & 101, 121, 116,  & 467, 592, 485, &    511.6       \\ 
$\ell=20$          &   & 146, 117  & 559, 455 &                \\[1mm]
\hline
$N=10$,            & $8.2178 \times 10^9$ & 913, 757, 735  & 3853, 2768, 2872, &    3233.4       \\ 
$\ell=40$          &  & 915, 664  & 3912, 2762 &                \\[1mm]
\hline
$N=15$,            & $1.8240 \times 10^{15}$   & 2575, 2401, 2342,  & 9901, 9301, 9141, &    10024.4       \\ 
$\ell=60$          &  & 2540, 2658  & 9967, 11812 &             \\[1mm]
\hline
$N=20$,            & $4.2879 \times 10^{20}$ & 5649, 5633, 4839,  & 24856, 24227, 20272, &    22898.4       \\ 
$\ell=80$          &  & 5557, 5256  & 22872, 22265 &             \\[1mm]
\hline
$N=25$,            & $1.0404 \times 10^{26}$  & 10611, 9436, 8334,  & 48188, 40939, 31371, &    40017.4       \\ 
$\ell=100$         &  & 10186, 9322  & 40724, 38860 &            \\[1mm]
\hline
$N=30$,            & $2.5759 \times 10^{31}$   & 15539, 14262, 15484,  & 65566, 58283, 75729, &    65068.2       \\ 
$\ell=120$         &  & 15190, 14714  & 63004, 62759 &            \\[1mm]
\hline
\end{tabular}
\end{center}
\end{tiny}
\begin{center}
\caption{Number of iterations (accepted columns) and samples (genetically generated columns) needed by {\tt GenCol} to find the exact ground state for MMOT with Coulomb cost and homogeneous marginal in one dimension. The number $\ell$ of gridpoints was increased proportionally to the number $N$ of electrons in line with \eqref{eq:thermo} (see left column) and 5 different runs were performed for each system.}
\end{center}
\label{T:thermo}
\end{table}

\begin{figure}[http!]
\begin{center}
\includegraphics[width=0.8\textwidth ]{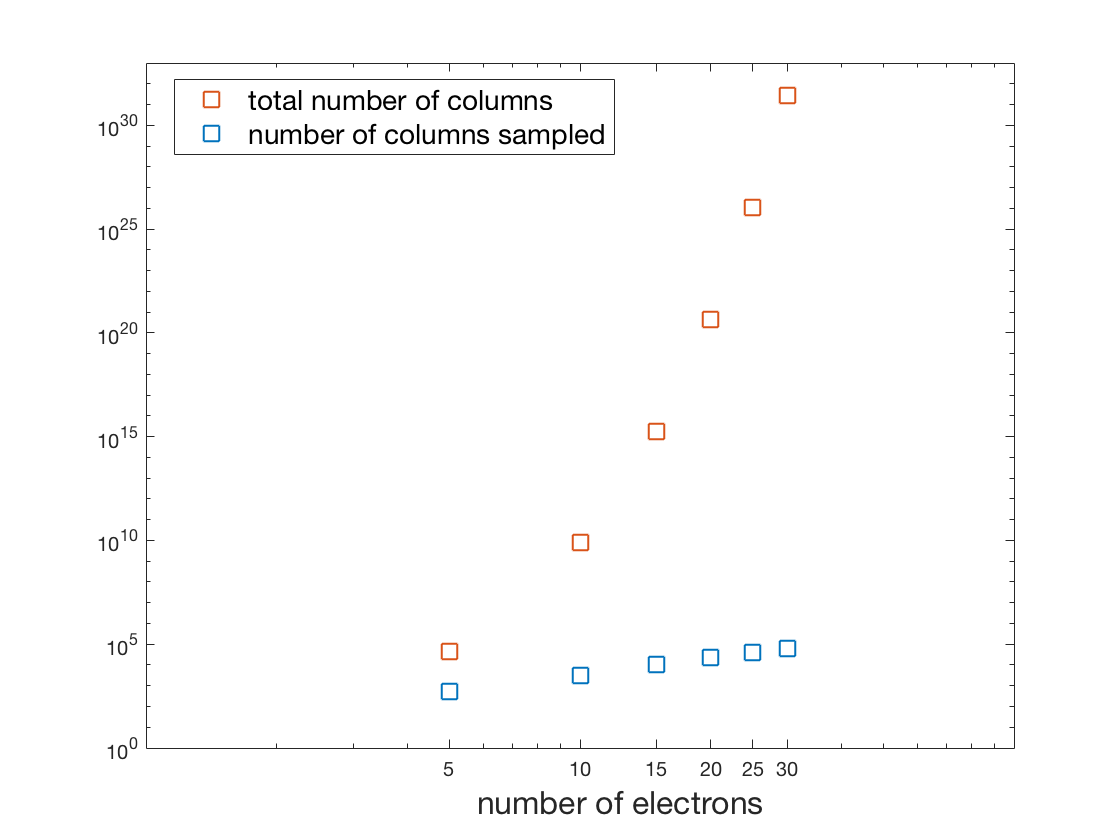}
\caption{Double logarithmic plot of the total number of genetic samples needed by {\tt GenCol} to find the exact ground state versus the number of electrons, for the systems described in Table 1. The plot gives the average number of samples over 5 runs (right column of the table) and the number of gridpoints was increased proportionally to the number of electrons (see the left column of the table).}
\end{center}
\label{F:breakcurse}
\end{figure}

%
%
%
%
\section{Discussion and conclusions}
The main advantage making our algorithm much faster than previous methods appears to be its simplicity: one just needs to solve low-dimensional LPs. Moreover after discretization no further approximations are made and the marginal constraints are automatically maintained, making the solution very accurate. Finally we note that the method also gives the Kantorovich potential, which is needed in applications to electronic structure.  

\begin{footnotesize}
 
\newcommand{\etalchar}[1]{$^{#1}$}

\end{footnotesize}

\end{document}